\documentclass{article}
\usepackage[american]{babel}
\usepackage{natbib}
\usepackage{mathtools}
\usepackage{amssymb}
\usepackage{amsfonts}
\usepackage{booktabs}
\usepackage{makecell}
\usepackage{enumitem}
\usepackage{url}
\usepackage{graphicx}
\usepackage{subcaption}
\usepackage{adjustbox}
\usepackage[toc,page]{appendix}
\usepackage{float}
\usepackage[top=2cm, bottom=2cm, left=2cm, right=2cm]{geometry}
\usepackage{authblk}

\usepackage{tikz}
\usepackage[most]{tcolorbox}
\tcbset{
  colback=gray!2,
  colframe=gray!35,
  arc=2mm,
  boxrule=0.4pt,
  left=6pt,right=6pt,top=6pt,bottom=6pt
}

\usepackage{algorithm}
\usepackage{algpseudocode}

\usepackage{amsthm}
\usepackage{thmtools}
\usepackage{hyperref}
\usepackage[capitalize,noabbrev]{cleveref}

\newtheorem{theorem}{Theorem}[section]
\newtheorem{lemma}[theorem]{Lemma}

\theoremstyle{definition}

\newtheorem{assumption}[theorem]{Assumption}

\theoremstyle{remark}
\newtheorem{remark}[theorem]{Remark}

\crefname{theorem}{theorem}{theorems}
\Crefname{theorem}{Theorem}{Theorems}
\crefname{lemma}{lemma}{lemmas}
\Crefname{lemma}{Lemma}{Lemmas}
\crefname{corollary}{corollary}{corollaries}
\Crefname{corollary}{Corollary}{Corollaries}
\crefname{definition}{definition}{definitions}
\Crefname{definition}{Definition}{Definitions}
\crefname{assumption}{assumption}{assumptions}
\Crefname{assumption}{Assumption}{Assumptions}
\crefname{example}{example}{examples}
\Crefname{example}{Example}{Examples}
\crefname{exercise}{exercise}{exercises}
\Crefname{exercise}{Exercise}{Exercises}
\crefname{remark}{remark}{remarks}
\Crefname{remark}{Remark}{Remarks}

\newcommand{\I}{{\mathbb{I}}}

\newcommand{\X}{\mathbf{X}}

\renewcommand{\P}{{\mathbb{P}}}

\newcommand{\ourmethod}{\texttt{CREDO}}

\title{CREDO: Epistemic-Aware Conformalized Credal Envelopes for Regression}

\author[1,2,4]{\href{mailto:lucruz45.cab@gmail.com}{Luben~M.~C.~Cabezas}}
\author[3]{Sabina J. Sloman}
\author[1]{Bruno M. Resende}
\author[3]{Fanyi Wu}
\author[3]{Michele Caprio}
\author[1]{Rafael Izbicki}

\affil[1]{%
    Department of Statistics\\
    Federal University of São Carlos\\
    São Carlos, Brazil
}
\affil[2]{%
    Institute of Mathematics and Computer Science \\ 
    University of São Paulo\\
    São Carlos, Brazil
}
\affil[3]{%
    The University of Manchester, UK\\
    Department of Computer Science
}
\affil[4]{%
    Université Grenoble Alpes, Inria, CNRS, Grenoble INP, LJK, France
}
  
\begin{document}
\maketitle

\begin{abstract}
Conformal prediction delivers prediction intervals with  distribution-free
coverage, but its intervals can look overconfident in regions where the model is
extrapolating, because standard conformal scores do not explicitly represent
 epistemic uncertainty. Credal methods, by contrast,
make epistemic effects visible by working with sets of plausible predictive
distributions, but they are typically model-based and lack calibration guarantees.
We introduce \texttt{CREDO}, a simple “credal-then-conformalize’’ recipe that
combines both strengths. \texttt{CREDO} first builds an interpretable credal
envelope that widens when local evidence is weak, then applies split conformal
calibration on top of this envelope to guarantee marginal coverage without further
assumptions. This separation of roles yields prediction intervals that are interpretable:  their width can be decomposed into aleatoric noise,
epistemic inflation, and a distribution-free calibration slack. We provide
a fast implementation based on trimming extreme posterior predictive
endpoints, prove validity, and show on benchmark regressions that \texttt{CREDO}
maintains target coverage while improving sparsity adaptivity at
competitive efficiency.
\end{abstract}

\section{Introduction}\label{sec:intro}
Uncertainty quantification (UQ) has become a central requirement for modern machine
learning systems, especially in settings where predictions inform downstream
decisions. In regression, a central task is to construct prediction intervals for a
future response $Y_{n+1}$ at covariates $X_{n+1}=x$.
Two complementary lines of work address this goal: conformal prediction (CP) provides distribution-free calibration, while imprecise-probability (credal) methods represent epistemic uncertainty via sets of plausible distributions but are typically model-based. In this paper, we combine these strengths to obtain prediction intervals that are both calibrated and epistemically interpretable.

On the calibration side, CP has emerged as one of the most influential approaches to
distribution-free predictive inference because it offers  distribution-free marginal coverage while assuming very
little about the data-generating mechanism \citep{Vovk2005,ShaferVovk2008,Lei2018}.
In its split form, conformal inference is simple and scalable: one fits any predictive model on a training set, computes a
nonconformity score on a calibration set, and then enlarges a model-based prediction set by
a data-driven correction chosen to ensure coverage.
However, standard conformal scores do not explicitly reflect \emph{epistemic} uncertainty---uncertainty due to limited information---as opposed to \emph{aleatoric} noise inherent in the data-generating process (see Section~\ref{sec:related} for related work).
As a result, split-conformal intervals can remain relatively narrow in regions of low local support where the underlying predictor extrapolates confidently; while marginal coverage is still guaranteed, the interval geometry may fail to communicate where predictions are driven by extrapolation rather than evidence.

From a different direction, the literature on imprecise probabilities and
credal sets provides a principled language for epistemic uncertainty by
representing beliefs not by a single distribution, but by a
set of plausible distributions.
Imprecise Probability (IP) theory \citep{Walley1991,augustin2014introduction,TroffaesDeCooman2014} studies uncertainty in how stochasticity is modeled.
Credal sets, convex and (weak$^\star$-)closed sets of probabilities \citep{levi2}, are among the main tools from IP theory. They are attractive precisely because they
can make epistemic effects explicit: in poorly informed regions the set of plausible conditionals can widen, while additional evidence can reduce this imprecision \citep{dilation,caprio2023constriction}. Yet, credal prediction
sets are typically \emph{model-based} and thus \emph{uncalibrated}: they depend on how the credal set is constructed and do not usually come with distribution-free coverage guarantees.

This paper aims to combine the best of both worlds in regression. We develop a
framework that (a) produces interpretable credal predictive uncertainty to represent local epistemic effects, and (b) retains the distribution-free marginal coverage guarantee of conformal calibration.
Concretely, we start from a local credal set of conditional predictive distributions $\mathcal F_0(x)$, induce from it a \emph{credal quantile envelope} $[\ell(x),u(x)]$ that summarizes the range of plausible conditional quantiles, and then apply split conformal calibration using a distance-to-envelope score.
The resulting prediction interval, \ourmethod{} (Conformalized Regression with Epistemic-aware creDal envelOpes), is easy to implement, retains distribution-free marginal coverage, and expands adaptively in regions where the credal envelope reflects limited local information.

Figure~\ref{fig:credo_vs_cqr} illustrates the motivation on a toy regression with a structural transition near $x=0$ and weaker covariate support toward the extremes. 
Both CQR \citep{Romano2019} and \ourmethod{} are split-conformal and thus target the same marginal coverage, but their \emph{geometry} differs: \ourmethod{} selectively widens where local evidence is weak (near the transition and at the edges), whereas CQR can remain comparatively uniform and appear overconfident in sparse regions. 
 \ourmethod{} also makes this behavior diagnosable via a decomposition of interval width into aleatoric and epistemic components, with epistemic uncertainty peaking exactly where the envelope inflates.
\begin{figure}[h]
    \centering
    \includegraphics[width=0.7\linewidth]{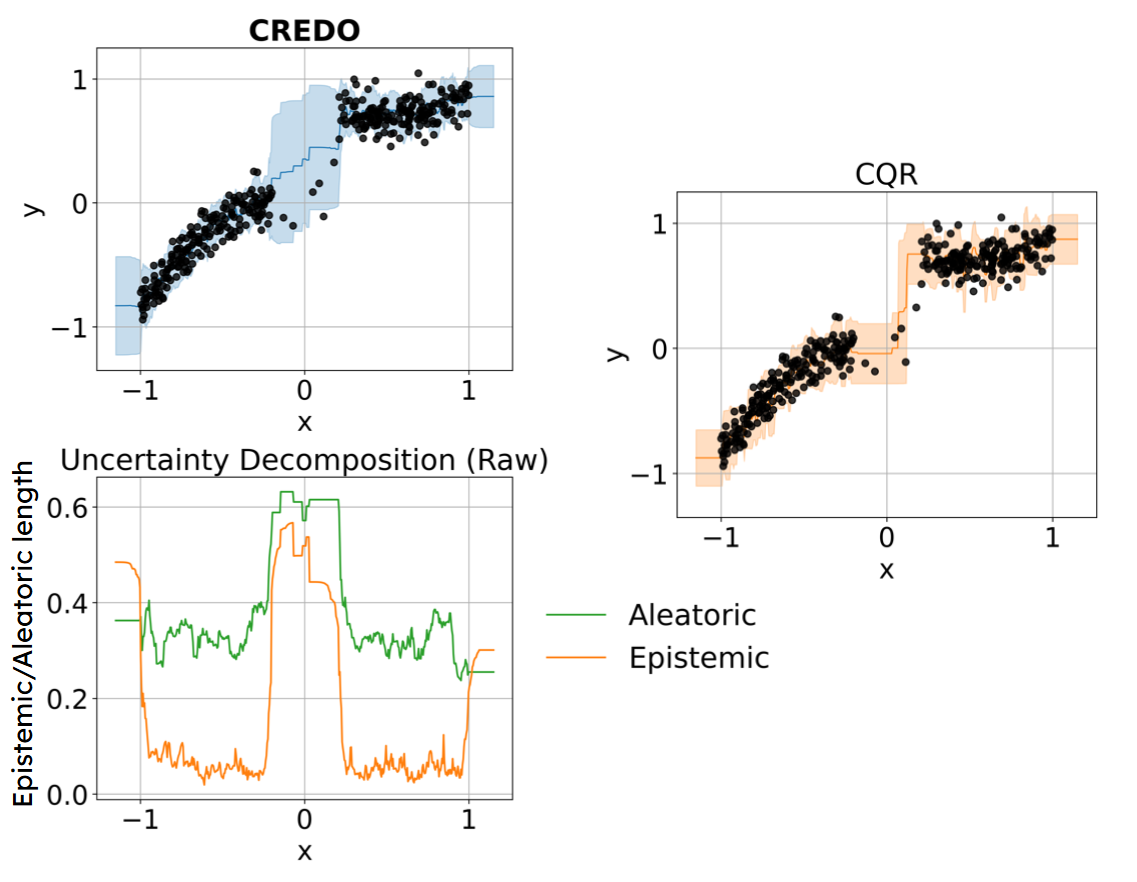}
\caption{\textbf{Motivation and decomposition.}
Comparison of split-conformal prediction intervals.
\textit{Top-left:} \ourmethod{}  yields intervals that widen adaptively in regions of high epistemic uncertainty—near the transition and toward the covariate extremes.
\textit{Top-right:} CQR (which does not take epistemic uncertainty into account) produces comparatively uniform intervals that can look overconfident where local support is weak (few nearby training points), despite marginal validity.
\textit{Bottom-left:} \ourmethod{} enables an uncertainty decomposition into aleatoric (green) and epistemic (orange) components; epistemic uncertainty peaks around $x=0$ and at the edges, aligning with where \ourmethod{} inflates the envelope. }
\label{fig:credo_vs_cqr}
\end{figure}

\subsection{Novelty}
Our contributions are:

\paragraph{Credal-to-conformal regression via an explicit envelope.}
We propose \ourmethod{}, which starts from a covariate-dependent credal set of predictive distributions $\mathcal F_0(x)$ and summarizes it by a $(1-\alpha_0)$ credal quantile envelope $[\ell(x),u(x)]$.
We then conformalize this envelope using a distance-to-envelope score, guaranteeing finite-sample marginal coverage under exchangeability.

\paragraph{A lightweight credal construction from posterior endpoint trimming.}
We introduce an endpoint-trimmed posterior credal set that produces envelopes by trimming extreme posterior predictive endpoints.
This yields a simple and scalable mechanism to encode epistemic effects without modifying the conformal machinery.

\paragraph{Diagnostics via a width decomposition.}
We provide a practical diagnostic that attributes interval width to (i) an aleatoric core from the conditional model, (ii) an epistemic inflation term induced by credalization, and (iii) the conformal calibration slack. This makes it possible to localize \emph{why} uncertainty is large at a given $x$.

\subsection{Relation to Other Work}
\label{sec:related}

\paragraph{Conformal Inference.}
Conformal prediction yields distribution-free predictive sets by calibrating a
nonconformity score $s(x,y)$ on held-out data and inverting an empirical
$(1-\alpha)$ quantile, guaranteeing finite-sample marginal coverage under
exchangeability \citep{Vovk2005,ShaferVovk2008,Papadopoulos2002,Lei2018,AngelopoulosBates2021,CabezasOttoIzbickiStern2025}.
From the standpoint of epistemic uncertainty, this guarantee does not by
itself ensure that sets expand in data-sparse or extrapolative regions: widely
used regression scores (including residual scores and CQR-style quantile-distance
scores) often track primarily aleatoric variability, so the resulting
expansion can be weakly sensitive to lack of local support
\citep{Lei2018,Romano2019,IzbickiShimizuStern2020,IzbickiShimizuStern2022,HullermierWaegeman2021,Izbicki2025}.

Motivated by this, a growing line of work augments conformal prediction with
epistemic-aware scores so that the conformal expansion increases when the
predictor is extrapolating. For example,
\citet{Cocheteux2025} modifies weighted regression-split conformal prediction by
redefining the scale term to represent epistemic uncertainty about $Y$, estimated
via Monte Carlo dropout. Closest to our regression setting, \citet{Rossellini2024}
propose uncertainty-aware CQR variants that inflate the CQR score using
ensemble-based proxies for epistemic instability in the estimated quantiles
(e.g., ensemble standard deviation in UACQR-S or ensemble order statistics in
UACQR-P). A complementary, score-centric’ route is EPICScore which models the
conditional distribution of a base score with Bayesian tools and then
transforms the score to make epistemic effects visible while
retaining conformal calibration \citep{CabezasEtAlEPICSCORE}, though it can require
large calibration samples to learn this conditional score distribution.
While these methods successfully inject epistemic information into conformal
calibration, the epistemic component typically enters as an implicit
score-rescaling and is therefore harder to interpret. In contrast,
\ourmethod{} yields calibrated
intervals whose width admits a transparent decomposition into aleatoric noise,
credal (epistemic) inflation, and calibration slack, enabling direct inspection
of where each type of uncertainty grows.

\paragraph{Bayesian Predictive Inference.}
Bayesian predictive inference constructs uncertainty regions by integrating
parameter uncertainty through the posterior distribution
\citep{bernardo2009bayesian}. In this framework, uncertainty arises from both
aleatoric variability and epistemic uncertainty about model parameters
\citep{HullermierWaegeman2021}. While posterior predictive intervals are
calibrated under correct model specification, they may substantially under-cover
under misspecification or distribution shift, since their validity depends on
the assumed likelihood model \citep{Jansen2013,ovadia2019}.

Several recent works combine Bayesian modeling with conformal prediction.
For example, \citet{FongHolmes2021} use posterior predictive densities as
nonconformity scores, and \citet{Wu2026BCP} optimize the conformal threshold
under a decision-risk problem via Bayesian Quadrature. Bayesian--conformal
hybrids also appear in the Gaussian-process literature, where GP posterior
uncertainty is leveraged to widen calibrated sets away from observed inputs
\citep{Jaber2024,PionVazquez2024}. Conformalized Bayesian Inference  
\citep{bariletto2025conformalized,cabezas2025cp4sbi} takes a different approach by applying
conformal prediction directly to Bayesian posterior distributions in parameter
space, yielding valid uncertainty regions for parameters and derived quantities.
 These approaches typically rely on a single posterior predictive distribution
that averages over parameter uncertainty. 

In contrast, our method encodes
epistemic uncertainty through \emph{credal sets} of predictive distributions.
This representation captures not only parameter uncertainty but also model
ambiguity, providing a principled way to inflate predictions when multiple
modeling explanations are compatible with the local data. As a result, the
induced inflation can be more robust in data-sparse or extrapolative regions,
where a Bayesian posterior may extrapolate confidently due to inductive bias or
approximation artifacts. We then conformalize the resulting ambiguity-aware
envelopes, retaining finite-sample marginal coverage under exchangeability.



\paragraph{Credal Sets and Imprecise Probabilities.} 
Rather than restricting to a single probability distribution, the theory of imprecise probabilities represents epistemic uncertainty by second order distributions, random sets, or credal sets \citep{Walley1991}. In this paper, we focus on the latter for its ease of implementation and interpretation.
Credal sets arise naturally when evidence is insufficient to uniquely identify a distribution, and provide a principled framework for separating epistemic from aleatoric uncertainty \citep{abellan,HullermierWaegeman2021}.
 Recent work has established formal connections between conformal prediction and imprecise probability. In particular, \citet{Cella2022} show that, under a minimal assumption, conformal p-values induce an upper envelope $\overline{P}(A)=\sup_{P \in \mathcal{P}}P(A)$, for all measurable sets $A$, for a credal set $\mathcal{P}$ of predictive distributions. 
 
 Building on this perspective, \citet{caprio2025conformalpredictionregionsimprecise,caprio2025joyscategoricalconformalprediction} prove that a classical conformal prediction region coincides with the imprecise highest density regions (IHDR, the IP version of a credible interval) of the credal set $\mathcal{P}$, thereby revealing a structural equivalence between conformal prediction and certain imprecise probabilistic constructions. Recent work further explores this connection from an epistemic uncertainty perspective. \citet{Chau2026MMICP} show that split conformal prediction implicitly induces a predictive credal set, and propose a mutual-model-information criterion to quantify epistemic uncertainty within this framework.
Relatedly, \citet{caprio2025conformalizedcredalregionsclassification} develop conformalized credal regions for classification with ambiguous ground truth, combining credal uncertainty with conformal validity guarantees.
Similarly, \citet{Alireza2024} introduce conformalized credal set predictors that leverage conformal calibration to construct distributionally robust prediction sets while retaining finite-sample validity.

Rather than interpreting conformal regions as arising from an induced credal set, we explicitly construct a covariate-dependent credal set of predictive distributions to model local epistemic uncertainty, and subsequently apply split conformal calibration. In this way, \ourmethod{} integrates credal modeling and conformal validity at the level of predictive distributions.

\section{Methods}
\label{sec:method}
Our goal is to construct prediction sets for a real-valued response $Y_{n+1}$
given covariates $X_{n+1}=x$ that
(i) encode \emph{local epistemic uncertainty} through a covariate-dependent
\emph{credal} set of predictive distributions, and
(ii) enjoy \emph{finite-sample, distribution-free marginal coverage}
via conformal calibration.
 Let $D=\{(X_1,Y_1),\ldots,(X_n,Y_n)\}$ be an exchangeable dataset.
We split $D$ into a training set $D_{\mathrm{tr}}$ and a calibration set
$D_{\mathrm{cal}}$ of size $m:=|D_{\mathrm{cal}}|$.
All model-based quantities (credal sets and envelopes) are learned using
$D_{\mathrm{tr}}$; the conformal correction is computed using $D_{\mathrm{cal}}$.

\paragraph{High-level idea.}
Fix a ``credal nominal'' level $\alpha_0\in(0,1)$ and a conformal level
$\alpha\in(0,1)$.
For each covariate value $x$, we:
\begin{enumerate}
\item  Construct a local credal set of conditional predictive
distributions $\mathcal F_0(x)$ for $Y\mid X=x$.
\item  Take the credal \emph{quantile envelope}
$[\ell(x),u(x)]$ induced by $\mathcal F_0(x)$ at level $1-\alpha_0$.
\item   Use the distance-to-envelope score
$s(x,y)$ and compute its calibration quantile $\hat\tau$ on $D_{\mathrm{cal}}$.
\item Output the conformalized interval
$C(x)=[\ell(x)-\hat\tau,\;u(x)+\hat\tau]$.
\end{enumerate}
The final interval can be interpreted as an \emph{aleatoric core} (from the
conditional model), plus an \emph{epistemic expansion} (from the credal set),
plus a \emph{distribution-free calibration expansion} (from conformal). Figure \ref{fig:credo_scheme} summarizes our procedure. 
Next, we detail each of these steps.

\begin{figure}[H]
    \centering
    \includegraphics[width=0.7\columnwidth]{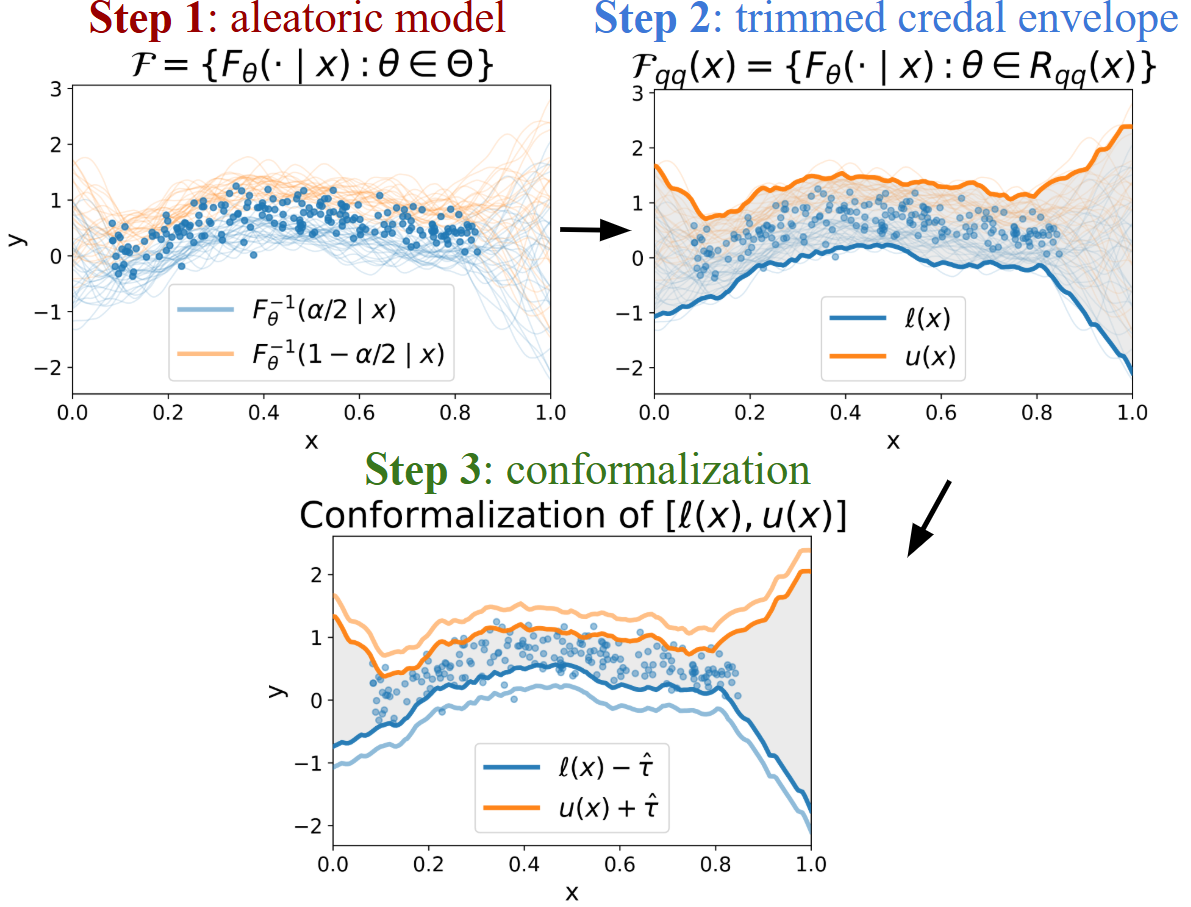}
    \caption{\ourmethod{} scheme.
From a model for the aleatoric uncertainty, we form a trimmed credal quantile envelope $[\ell(x),u(x)]$ (at nominal level $1-\alpha_0$), then split-conformalize using the distance-to-envelope score to output $C(x)=[\ell(x)-\hat\tau,\;u(x)+\hat\tau]$ with marginal coverage $1-\alpha$.}

    \label{fig:credo_scheme}
\end{figure}

\paragraph{Step 1: conditional aleatoric model and local credal sets.}
We start from a (possibly nonparametric) conditional model class $
\mathcal F=\{F_\theta(\,\cdot\,\mid x):\theta\in\Theta\},$
where $F_\theta(\cdot\mid x)$ denotes the conditional CDF of $Y$ given $X=x$.
The likelihood model $F_\theta$ is assumed precise; epistemic uncertainty
is represented through uncertainty about $\theta$.
We define a local credal predictive set at covariate value $x$ to be any subset
$\mathcal F_0(x)\subseteq \mathcal F$, interpreted as the collection of
predictive distributions deemed plausible at $x$ after observing $D_{\mathrm{tr}}$. In Section \ref{sec:endpoint} we discuss how to choose $\mathcal F_0(x)$ in this paper, although our framework is  general.

\paragraph{Step 2: credal quantile envelope.}
Fix $\alpha_0\in(0,1)$. For each covariate value $x$,  we define the credal envelope
\begin{equation}
\label{eq:quantile-envelope}
\begin{aligned}
[\ell(x),u(x)]
&=
\Big[
\inf_{F\in\mathcal F_0(x)} F^{-1}(\alpha_0/2\mid x),
\sup_{F\in\mathcal F_0(x)} F^{-1}(1-\alpha_0/2\mid x)
\Big].
\end{aligned}
\end{equation}
By construction, for every $F\in\mathcal F_0(x)$, the model-based central
interval $\big[F^{-1}(\alpha_0/2\mid x), F^{-1}(1-\alpha_0/2\mid x)\big]$
is contained in $[\ell(x),u(x)]$.

\begin{remark}
The credal set $\mathcal F_0(x)$ induces a predictive probability box (p-box, that can be thought of as an interval of CDF's) \citep[Section 4.6.4]{augustin2014introduction}
via the lower/upper CDF bounds $\underline F(t\mid x):=\inf_{F\in\mathcal F_0(x)} F(t\mid x)$ and $\overline F(t\mid x):=\sup_{F\in\mathcal F_0(x)} F(t\mid x).$ 
Using the generalized inverse $G^{-1}(p\mid x):=\inf\{t:\,G(t\mid x)\ge p\}$,
the quantile envelope \eqref{eq:quantile-envelope} depends only on this p-box, 
as $\ell(x)=\overline F^{-1}(\alpha_0/2\mid x)$
and
$u(x)=\underline F^{-1}(1-\alpha_0/2\mid x).$
\end{remark}

\paragraph{Steps 3--4: conformal calibration of the credal envelope.}
The credal envelope $[\ell(x),u(x)]$ is not guaranteed to be calibrated in
finite samples. To obtain distribution-free validity, we apply split conformal
calibration using the nonconformity score
\begin{equation*}
\label{eq:score}
s(x,y)
:= d\!\left(y,[\ell(x),u(x)]\right)
= \max\{\ell(x)-y,\;y-u(x)\},
\end{equation*}
which measures how far $y$ lies outside the credal envelope \citep{Romano2019}.
Let $\hat\tau$ be the $\lceil (m+1)(1-\alpha)\rceil$-th order statistic of the
calibration scores $\{s(x_i,y_i):(x_i,y_i)\in D_{\mathrm{cal}}\}$.
The final prediction set at covariate value $x$ is
\begin{equation}
\label{eq:final-set}
C(x)
=\{y:\ s(x,y)\le \hat\tau\}
=
[\ell(x)-\hat\tau,\;u(x)+\hat\tau].
\end{equation}

\subsection{One implementation of Step 1: endpoint-trimmed posterior credal sets}
\label{sec:endpoint}
We now specify the choice of $\mathcal F_0(x)$ used in the experiments in this paper.
The key idea is to represent epistemic uncertainty by retaining only those
posterior predictive distributions whose central predictive endpoints
are not extreme under the posterior. This yields an interpretable and
computationally light credal set.

In order to do that, we adopt a Bayesian perspective and place a prior $\pi(\theta)$ on $\Theta$,
leading to a posterior $\pi(\theta\mid D_{\mathrm{tr}})$.
Fix $\alpha_0\in(0,1)$ and define the predictive endpoints under parameter
value $\theta$:
\begin{align*}
q_L(\theta,x)&=F_\theta^{-1}(\alpha_0/2\mid x) \text{ and}\\q_U(\theta,x)&=F_\theta^{-1}(1-\alpha_0/2\mid x).
\end{align*}
Under $\theta\sim\pi(\theta\mid D_{\mathrm{tr}})$, the pair
$\big(q_L(\theta,x),q_U(\theta,x)\big)$ is random.
Fix a trimming level $\gamma\in(0,1)$ and set $C_L(x)$ and $C_U(x)$ to satisfy
\begin{align*}
\pi\!&\left(q_L(\theta,x)\le C_L(x)\mid x,D_{\mathrm{tr}}\right)\\
=    \pi&\!\left(q_U(\theta,x)\ge C_U(x)\mid x,D_{\mathrm{tr}}\right)=\gamma/2.
\end{align*}
Intuitively, this discards the $\gamma/2$ lowest lower endpoints and the
$\gamma/2$ highest upper endpoints under the posterior.
We define the endpoint-trimmed parameter set as
\[
R_{qq}(x)
=
\{\theta:\ q_L(\theta,x)\ge C_L(x),\;\; q_U(\theta,x)\le C_U(x)\},
\]
and the induced local credal set $
\mathcal F_{qq}(x)=\{F_\theta(\cdot\mid x):\theta\in R_{qq}(x)\}.
$ 
By the union bound, notice that the posterior probability of $R_{qq}(x)$ is at least $1-\gamma$. 
Finally, the credal envelope with respect to this credal set is the
\emph{trimmed endpoint envelope}
\begin{equation}
\label{eq:trimmed-envelope}
[\ell(x),u(x)]:=[C_L(x),C_U(x)].
\end{equation}

This envelope is easy to implement via Monte Carlo  and captures posterior endpoint dispersion in a way that is directly
interpretable as epistemic uncertainty. 
Indeed, 
 in practice, we draw $\theta^{(1)},\ldots,\theta^{(B)}\sim\pi(\theta\mid D_{\mathrm{tr}})$,
compute endpoint draws $q_L^{(b)}(x),q_U^{(b)}(x)$, and set $C_L(x),C_U(x)$ to the
empirical $(\gamma/2)$-quantile of $\{q_L^{(b)}(x)\}_{b=1}^B$ and the empirical
$(1-\gamma/2)$-quantile of $\{q_U^{(b)}(x)\}_{b=1}^B$, respectively.
Section \ref{sec:alg} in the Appendix shows the full \ourmethod{} algorithm.

\subsection{Data-Density-Aware Credal Sets}

Posterior predictive intervals from common Bayesian learners can be too
insensitive to local data scarcity, yielding overconfident sets in
data-sparse regions unless epistemic structure is injected explicitly
\citep{ovadia2019}. For instance,
BART is built from regression trees with constant leaf
values, hence its fits are piecewise constant in $x$ \citep{Chipman2010BART},
and standard tree extrapolation can become flat outside the training range
\citep{Wang2022LocalGPBART}. Furthermore, a global conformal correction may fail to adequately account for heteroscedasticity even in highly populated regions \citep{Lei2018, Romano2019}. While a uniform shift $\hat{\tau}$ might be necessary to cover concentrated aleatoric noise in some areas, applying the same slack to well-supported, low-noise regions leads to overly conservative intervals that ignore the model's localized precision.

These limitations motivate using a
covariate-dependent $\gamma(x)$ that adapts to local information density.
By replacing the fixed $\gamma$ with $\gamma(x)$, we selectively retain more posterior mass—yielding wider credal envelopes—where covariate support is low and the model is prone to overconfidence. Conversely, in regions with high data density, $\gamma(x)$ allows for thinner envelopes, leveraging localized precision to reflect greater confidence in the model’s estimates. 


Concretely, for each $x$, we define $C_L(x)$ and $C_U(x)$ as before but with $\gamma$ replaced by $\gamma(x)$. A smaller $\gamma(x)$ retains more posterior realizations, typically yielding a wider envelope to capture epistemic uncertainty, whereas a larger $\gamma(x)$ trims more aggressively to produce efficient, tight intervals in well-supported areas. Importantly, because $\gamma(x)$ depends only on the covariates $X$, the split conformal validity of Equation \eqref{eq:final-set} is preserved.

In our implementation, we  choose $\gamma(x)$
according to
\[
\gamma(x)=\gamma_{\max}-(\gamma_{\max}-\gamma_{\min})\,
\sigma\!\left(\frac{\mathrm{sc}(x)-m_{\gamma}}{\tau_\gamma}\right),
\]
where $\mathrm{sc}(x)$ is a
scarcity score, $m_{\gamma}, \gamma_{\min}, \tau_\gamma,$
and $\gamma_{\max}$ are constants and $\sigma(x)$ is the logistic sigmoid function. The idea is that, if $x$ is in a data-sparse region,
$\gamma(x)$ will be small (close to $\gamma_{\min})$, leading to less trimming on the credal set construction (i.e., larger sets).
On the other hand, in well-supported regions,  $\gamma(x) \approx \gamma_{\max}$. 

To construct the scarsity score, we start with $\phi:\mathcal X\to\mathbb
R^d$, a representation in which Euclidean distances are meaningful (for
tabular data, $\phi(x)$ can be standardized covariates; for complex inputs, it
can be a learned embedding \citep{wilson2016deep} from the predictive model). Fix $k$ and, for any
$x$, let $X_{(k)}(x)$ denote the $k$-th nearest neighbor of $x$ among the
training covariates $\{X_i:(X_i,Y_i)\in D_{\mathrm{tr}}\}$. Define the kNN radius $
r_k(x):=\|\phi(x)-\phi(X_{(k)}(x))\|,
$ so larger $r_k(x)$ indicates lower local design density around $x$. To make
$r_k(x)$ comparable across problems, we compute reference quantiles on the
training set, e.g.
$q_{\mathrm{lo}}=\operatorname{Quantile}_{0.50}(\{r_k(X_i)\})$ and
$q_{\mathrm{hi}}=\operatorname{Quantile}_{0.95}(\{r_k(X_i)\})$, and form the
scarcity score
\begin{equation}
\label{eq:scarcity_score}
    \mathrm{sc}(x):=\frac{r_k(x)-q_{\mathrm{lo}}}{q_{\mathrm{hi}}-q_{\mathrm{lo}}+\varepsilon}.
\end{equation}

where $\varepsilon>0$ is small. We detail standard choices for each parameter and representations in Appendix \ref{appendix:adaptive_gamma}.


Figure~\ref{fig:gamma} shows the data-support--aware trimming level $\gamma(x)$ as a function of $x$ in a toy example. In well-supported regions, $\gamma(x)$ is larger, so more aggressive endpoint trimming keeps the credal envelope $[\ell(x),u(x)]$ close to the baseline posterior interval. In data-sparse or extrapolative regions, $\gamma(x)$ becomes smaller, reducing trimming and thereby widening $[\ell(x),u(x)]$ to reflect increased epistemic uncertainty.

\begin{figure}[h]
    \centering
    \includegraphics[width=0.7\linewidth]{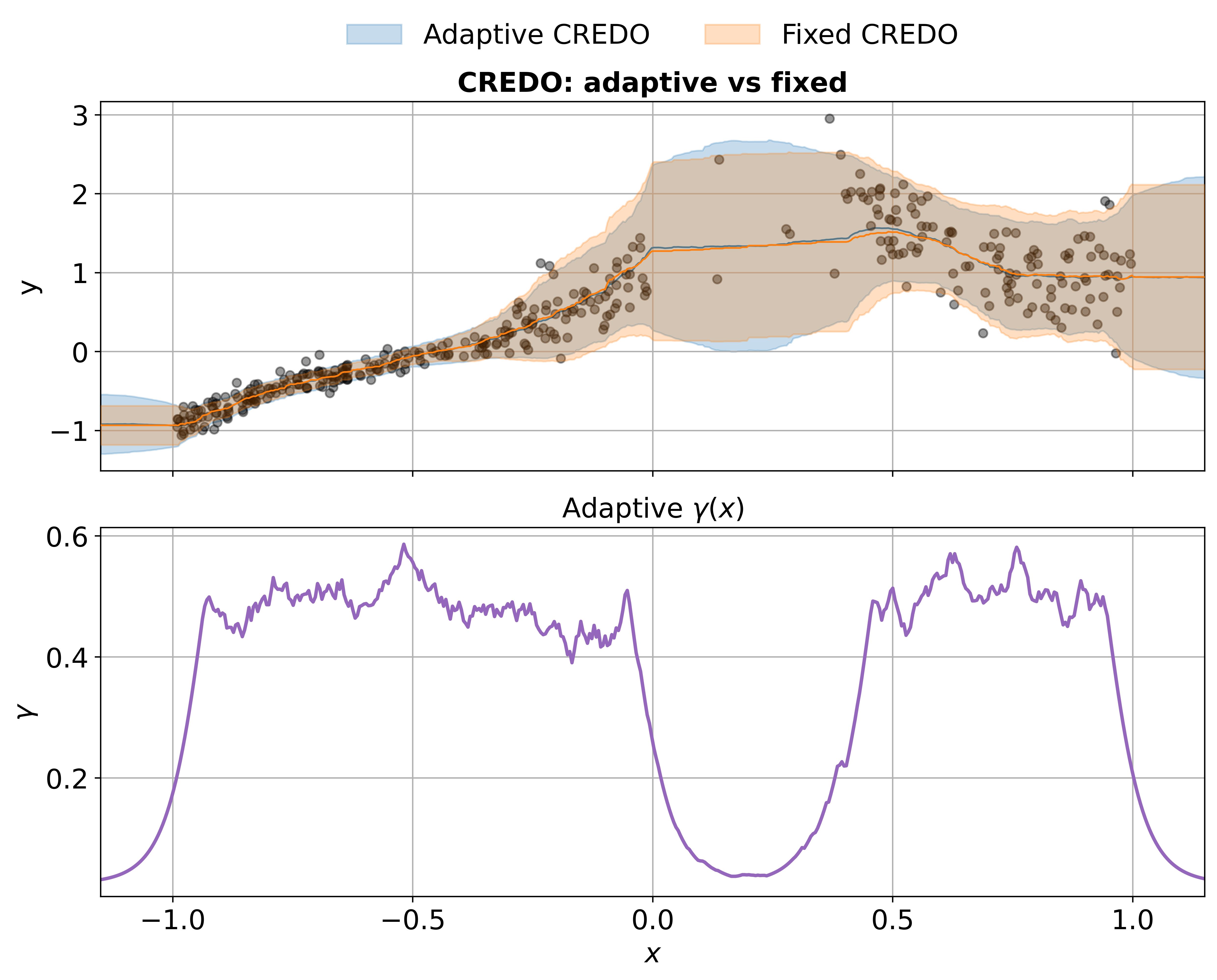}
  \caption{\textbf{Data-support--aware trimming via $\gamma(x)$.}
\textit{Top:} Prediction intervals under a constant trimming level $\gamma$ versus the adaptive rule $\gamma(x)$ driven by a kNN-based scarcity score.
Adaptive trimming decreases $\gamma(x)$ in data-sparse/extrapolative regions (less endpoint trimming, wider credal envelopes) and increases $\gamma(x)$ in data-dense regions (more trimming, tighter envelopes).
\textit{Bottom:} The resulting $\gamma(x)$ profile as a function of the covariate $x$.
}
    \label{fig:gamma}
\end{figure}

\subsection{Uncertainty Decomposition}
\label{sec:unc_decomp}

A central feature of \ourmethod{} is that it yields an
interpretable decomposition of the final conformal interval length into (i) an
aleatoric baseline coming from the conditional model, (ii) an epistemic
contribution induced by posterior uncertainty over predictive endpoints, and
(iii) a distribution-free calibration adjustment.

For each parameter value $\theta$,  we define the conditional (aleatoric) central  
interval at covariate $x$,
$
I_\theta(x)=\bigl[q_L(\theta,x),\ q_U(\theta,x)\bigr],$ 
and its length $|I_\theta(x)|=q_U(\theta,x)-q_L(\theta,x)$. We summarize the
irreducible (aleatoric) uncertainty at $x$ by the posterior mean length
\begin{align*}
U_A(x) :&=\mathbb E\!\left[|I_\theta(x)|\mid x, D_{\mathrm{tr}}\right]\\
&=\int \bigl(q_U(\theta,x)-q_L(\theta,x)\bigr)\,d\pi(\theta\mid D_{\mathrm{tr}}).    
\end{align*}
We then decompose the total length of 
$C(x)$ (Eq.~\ref{eq:final-set}) as
\begin{equation*}
\begin{aligned}
|C(x)|
&= \underbrace{U_A(x)}_{\mathclap{\text{aleatoric baseline}}}\quad + \underbrace{\Bigl(|[\ell(x),u(x)]|-U_A(x)\Bigr)}_{\mathclap{\text{epistemic contribution}}} \\
&+ \underbrace{2\hat\tau}_{\mathclap{\text{distribution-free calibration slack}}}.
\end{aligned}
\end{equation*}

In practice, $U_A(x)$ is  estimated via Monte Carlo draws
$\theta^{(1)},\ldots,\theta^{(B)}\sim \pi(\theta\mid D_{\mathrm{tr}})$:
\[
\widehat U_A(x)=\frac{1}{B}\sum_{b=1}^B \bigl(q_U(\theta^{(b)},x)-q_L(\theta^{(b)},x)\bigr),
\]
while $|[\ell(x),u(x)]|$ is obtained directly from the trimmed endpoint
quantiles, and $\hat\tau$ is computed from the calibration scores
\eqref{eq:score}. 

\section{Theory}

We begin by showing that the credal quantile envelope defined in Eq.~\ref{eq:quantile-envelope} is conservative for every distribution in the local credal set. Concretly,
 for any fixed
  $x$ and any $F\in\mathcal F_0(x)$, the credal envelope covers
$Y\mid X=x$ with probability at least $1-\alpha_0$ (according to $F$): 

\begin{theorem}
\label{thm:credal-coverage-envelope}
Assume $[\ell(x),u(x)]$ is defined by \eqref{eq:quantile-envelope}.
Then for every fixed $x$ and every $F\in\mathcal F_0(x)$,
\[
\mathbb{P}_{F}( Y\in[\ell(x),u(x)] \mid X=x)\ \ge\ 1-\alpha_0.
\]
\end{theorem}

 The next theorem
provides a posterior-predictive lower bound on the coverage probability, showing that trimming at level
$\gamma$ yields a Bayesian predictive mass of at least
$(1-\gamma)(1-\alpha_0)$ inside the trimmed envelope. These results characterize
the credal step as a principled (and interpretable) enlargement of a nominal
conditional interval.

\begin{theorem}
\label{theorem:posterior-predictive-trimmed}
Fix $x\in\mathcal X$ and consider the endpoint-trimmed construction in
Section~\ref{sec:endpoint} at nominal level $\alpha_0\in(0,1)$ and trimming
level $\gamma\in(0,1)$, with
$[\ell(x),u(x)]=[C_L(x),C_U(x)]$ as in Eq.~\eqref{eq:trimmed-envelope}.
Define the (Bayesian) posterior predictive probability by
\[
\P(\,\cdot\,\mid X=x, D_{\mathrm{tr}})
:= \int \P_\theta(\,\cdot\,\mid X=x)\, d\pi(\theta\mid D_{\mathrm{tr}}),
\]
where $\P_\theta(\cdot\mid X=x)$ denotes probability under the conditional
model $F_\theta(\cdot\mid x)$.
Then
\[
\P\!(Y\in[\ell(x),u(x)] \mid X=x, D_{\mathrm{tr}})
\ \ge\ (1-\gamma)(1-\alpha_0).
\]
\end{theorem}

The credal envelope alone, however, is not guaranteed to be calibrated under any data-generating process. However, its
conformalized, \ourmethod, version does achieve the correct marginal coverage:

\begin{theorem}
\label{thm:split-conformal}
Let $(X_1,Y_1),\dots,(X_n,Y_n),(X_{n+1},Y_{n+1})$ be exchangeable. 
Then the \ourmethod{}  set
$
C(X_{n+1})=[\ell(X_{n+1})-\hat\tau,\ u(X_{n+1})+\hat\tau]$
satisfies the distribution-free guarantee
\[
\mathbb{P}\bigl(Y_{n+1}\in C(X_{n+1})\bigr)\ge 1-\alpha.
\]
\end{theorem}

 Finally, under
correct specification and posterior consistency, we show that this conformal
correction is asymptotically benign: \Cref{thm:oracle} proves that the trimmed
envelope endpoints converge to the oracle conditional quantiles, the conformal
expansion $\hat\tau$ vanishes, and the conditional coverage converges to the
nominal level. Together, these theorems justify \ourmethod{} as combining
(i) explicit, covariate-dependent epistemic uncertainty through credal sets with
(ii) finite-sample distribution-free validity through conformal calibration,
while recovering the oracle interval asymptotically under standard Bayesian
regularity.

\begin{theorem}
\label{thm:oracle}
Suppose Assumption~\ref{assump:oracle} holds (in particular, $\alpha_0=\alpha$).
 Construct the endpoint-trimmed envelope as in Section~\ref{sec:endpoint}:
$[\ell(x),u(x)]=[C_L(x),C_U(x)]$, where $C_L(x)$ is the $(\gamma(x)/2)$-posterior
quantile of $q_L(\theta,x)$ and $C_U(x)$ is the $(1-\gamma(x)/2)$-posterior
quantile of $q_U(\theta,x)$, and let
$C(x)=[\ell(x)-\hat\tau,\ u(x)+\hat\tau]$ be the \ourmethod{} interval.
Then, as $n_{\mathrm{tr}}\to\infty$ and $m\to\infty$, for $P_X$-almost every $x$,
\begin{enumerate}
\item[(i)] \textbf{Envelope endpoints converge to the oracle endpoints:}
\[
(\ell(x),u(x))
\ \xrightarrow{\ \mathbb P_{\theta^\star}\ }\ 
(q_L^\star(x),q_U^\star(x)),
\]
where
$
q_L^\star(x)=F_{\theta^\star}^{-1}(\alpha/2\mid x)$ and
$q_U^\star(x)=F_{\theta^\star}^{-1}(1-\alpha/2\mid x).$

\item[(ii)] \textbf{The conformal correction vanishes:} $
\hat\tau\ \xrightarrow{\ \mathbb P_{\theta^\star}\ }\ 0.
$

\item[(iii)] \textbf{Conditional coverage converges to the nominal level:}
\[
\mathbb P_{\theta^\star}\!\bigl(Y\in C(x)\mid X=x,D_{\mathrm{tr}},D_{\mathrm{cal}}\bigr)
\ \xrightarrow{\ \mathbb P_{\theta^\star}\ }\ 1-\alpha.
\]
\end{enumerate}
\end{theorem}

\section{Experiments}
\label{sec:experiments}
We evaluate \ourmethod{} across 12 standard regression benchmarks (see Appendix~\ref{appendix:additional_res} for dataset details). Each experiment follows a randomized protocol, splitting the data into training (56\%), calibration (24\%), and test (20\%) sets, with performance metrics averaged over 30 independent runs. All evaluated methods target a 90\% marginal coverage level ($\alpha=0.1$). For \ourmethod{}, we initialize the credal level at $\alpha_0=\alpha$ and construct trimmed endpoint envelopes derived from posterior samples.

Our implementation uses a Bayesian Quantile Neural Network (QNN) backend, where we evaluate two primary variants: CREDO-QNN with a fixed credal level and an adaptive variant employing our covariate-dependent $\gamma(x)$ scheme. Technical implementation details for both variants are provided in Appendix~\ref{appendix:credo_details}. The remainder of this section is structured as follows: Section~\ref{sec:baseline_comparisons} provides a comparative analysis against state-of-the-art conformal prediction methods, focusing on conditional coverage efficiency and epistemic awareness. Subsequently, Section~\ref{sec:decomposition_evaluation} validates the interpretability of \ourmethod{} by evaluating the behavior of our decomposed epistemic uncertainty across inliers and outliers.



\subsection{Baseline comparisons}
\label{sec:baseline_comparisons}

To evaluate \ourmethod{}, we compare it against the following state-of-the-art conformal quantile regressors: standard CQR and CQR-r \citep{Romano2019, sesia2020comparison}, the epistemic-aware UACQR-S and UACQR-P \citep{Rossellini2024}, and EPICScore \citep{CabezasEtAlEPICSCORE}. These baselines span from classical split conformal methods based on quantile regression to more recent approaches that explicitly incorporate epistemic. Further implementation details are provided in Appendix~\ref{appendix:baseline_details}.

Performance is assessed via two primary objectives: conditional coverage efficiency and adaptivity to model ambiguity. We measure the former using the Scaled Mean Interval Score (SMIS), which jointly penalizes interval width and miscoverage \citep{GneitingRaftery2007} to serve as a proxy for the conditional coverage properties of the sets. To evaluate response to data scarcity, we report the Interval Length Ratio (ILR)—the ratio of mean interval lengths for test-set outliers relative to inliers \citep{CabezasEtAlEPICSCORE}. A higher ILR distinguishes methods that effectively inflate uncertainty in regions with low covariate support, reflecting a more sensitive capture of local epistemic signals. Extended results, metric definitions and outlier detection details are provided in Appendix~\ref{appendix:additional_res}.


Across all benchmarks, both \ourmethod{} variants achieve the target 90\% marginal coverage, confirming split-conformal validity in practice (Table \ref{tab:marginal_coverage}, Appendix \ref{appendix:additional_comparisons}). Regarding SMIS (Table~\ref{tab:aisl}), \ourmethod{} remains consistently competitive, often outperforming or statistically matching strong conformal baselines; specifically, the adaptive variant is a top performer in 9 out of 12 datasets.

\ourmethod{} also demonstrates superior outlier adaptivity. ILR results (Table~\ref{tab:ratio_intervals}) show that both variants consistently produce the highest ratios, indicating effective selective widening on atypical points without sacrificing global efficiency or coverage. This effect is most pronounced in small-data regimes, where increased posterior dispersion in sparse regions allows \ourmethod{} to better capture local epistemic uncertainty. Furthermore, \ourmethod{} maintains outlier coverage closer to nominal levels in the majority of scenarios (Table \ref{tab:coverage_outlier}, Appendix \ref{appendix:additional_comparisons}), further highlighting its enhanced epistemic awareness.

\subsection{Uncertainty decomposition results}
\label{sec:decomposition_evaluation}
As discussed in Section~\ref{sec:unc_decomp}, \ourmethod{} provides a natural disentanglement of uncertainty components, yielding a layer of interpretability over other conformal methods. To validate this decomposition, we compare the proportion of epistemic uncertainty—defined as the ratio of epistemic inflation to total non-conformalized interval width—between identified outliers and the 10\% most central inliers across all 12 benchmarks.

We calculate the average normalized epistemic uncertainty per observation over 30  independent repetitions. The results, visualized via the boxplots in Figure~\ref{fig:credo_boxplots_epistemic} (Appendix~\ref{appendix:decomp_real_data}), demonstrate a consistent trend: outliers across diverse datasets tend to exhibit a higher proportion of epistemic uncertainty compared to inliers. This separation is particularly pronounced in six of the twelve datasets—four of which are small-data regimes with fewer than 10,000 samples—where data-scarce regions naturally harbor greater model ambiguity.

\begin{table*}[h]
\centering
\caption{\textbf{Efficiency of predictive sets as measured by SMIS (lower is better).}
SMIS summarizes the coverage--width trade-off by penalizing both interval length and misses outside the interval.
Parentheses show $2\times$ the standard deviation.
Bold marks the best method within a 95\% confidence interval (ties allowed).
The sample size $n$ of each dataset is shown next to its name.
\ourmethod{} is consistently competitive on SMIS  and the adaptive variant (with $\gamma(x)$) is often best, suggesting improved efficiency without sacrificing calibration.}
\label{tab:aisl}
\begin{adjustbox}{max width=\textwidth}
\begin{tabular}{cccccccc}
\hline
\textbf{Dataset (n)}          & \textbf{CREDO}      & \textbf{CREDO (adap.)} & \textbf{CQR}            & \textbf{CQR-r}          & \textbf{UACQRS}         & \textbf{UACQRP}         & \textbf{EPIC}           \\ \hline
concrete (1030)                  & \textbf{27.59 (0.75)}   & \textbf{26.49 (0.73)}      & 28.56 (0.92)            & 35.49 (1.55)            & 36.52 (1.42)            & 40.94 (2.06)            & \textbf{27.51 (0.82)}   \\
airfoil (1503)                  & 11.37 (0.35)            & 11.25 (0.30)               & 10.51 (0.26)            & 13.68 (0.69)            & 12.68 (0.46)            & 14.74 (0.72)            & \textbf{9.95 (0.24)}    \\
winewhite (1599)                & 3.01 (0.03)             & \textbf{2.88 (0.03)}       & 2.96 (0.03)             & 2.99 (0.03)             & \textbf{2.92 (0.03)}    & 3.01 (0.03)             & \textbf{2.89 (0.03)}    \\
star $\times 10^2$ (2161)       & 9.92 (0.14)             & \textbf{9.75 (0.11)}       & 10.08 (0.17)            & 11.33 (0.30)            & \textbf{9.56 (0.14)}    & \textbf{9.68 (0.17)}    & 10.06 (0.17)            \\
winered (4898)                  & \textbf{2.74 (0.05)}    & \textbf{2.66 (0.05)}       & 2.98 (0.07)             & 3.14 (0.09)             & \textbf{2.77 (0.06)}    & 2.82 (0.06)             & 2.82 (0.06)             \\
cycle (9568)                    & 18.27 (0.19)            & 17.65 (0.14)               & 15.30 (0.16)            & 15.39 (0.17)            & 15.63 (0.16)            & 16.30 (0.22)            & \textbf{15.29 (0.17)}   \\
electric $\times 10^{-2}$ (10000) & \textbf{3.69 (0.16)}    & \textbf{3.52 (0.14)}       & 4.41 (0.10)             & 4.71 (0.11)             & 5.44 (0.18)             & 6.01 (0.24)             & 4.30 (0.09)             \\
meps19 (15781)                   & 61.77 (1.40)            & \textbf{57.21 (1.50)}      & \textbf{58.82 (2.09)}   & \textbf{58.59 (2.07)}   & \textbf{56.45 (1.85)}   & \textbf{58.82 (2.04)}   & \textbf{57.91 (2.04)}   \\
superconductivity (21263)        & 49.76 (0.54)            & \textbf{44.17 (0.51)}      & \textbf{43.76 (0.35)}   & \textbf{43.79 (0.34)}   & 45.26 (0.39)            & 47.57 (0.36)            & \textbf{43.67 (0.34)}   \\
homes $\times 10^4$ (21613)      & 46.19 (0.46)            & 43.21 (0.37)               & 42.81 (0.51)            & \textbf{42.60 (0.48)}   & 43.25 (0.50)            & 44.33 (0.57)            & \textbf{41.72 (0.47)}   \\
protein (45730)                  & 15.22 (0.05)            & \textbf{14.74 (0.06)}      & 15.17 (0.05)            & 15.17 (0.05)            & 15.33 (0.05)            & 15.77 (0.05)            & 15.12 (0.05)            \\
WEC $\times 10^4$ (54000)        & 9.67 (0.08)             & \textbf{8.67 (0.07)}       & 12.67 (0.12)            & 12.67 (0.12)            & 12.67 (0.12)            & 13.53 (0.12)            & 12.24 (0.11)            \\ \hline
\end{tabular}
\end{adjustbox}
\end{table*}

\begin{table*}[h]
\centering
\caption{\textbf{Adaptivity to outliers via the interval-length ratio (ILR; higher is better).}
ILR is the ratio of the average prediction-interval length on test-set outliers to the average length on test-set inliers. 
 Parentheses show $2\times$ the standard deviation.
Bold marks the best method within a 95\% confidence interval (ties allowed).
The sample size $n$ of each dataset is shown next to its name.  \ourmethod{} typically achieves higher ILR, with the largest gains on small-$n$ datasets where epistemic uncertainty is most pronounced, indicating more targeted widening when data are scarce.}
\label{tab:ratio_intervals}
\begin{adjustbox}{max width=\textwidth}
\begin{tabular}{cccccccc}
\hline
\textbf{Dataset (n)}  & \textbf{CREDO}     & \textbf{CREDO (adap.)} & \textbf{CQR}           & \textbf{CQR-r}         & \textbf{UACQRS}        & \textbf{UACQRP}        & \textbf{EPIC}          \\ \hline
concrete (1030)          & \textbf{1.09 (0.07)} & \textbf{1.15 (0.07)}     & 0.83 (0.05)          & 0.73 (0.07)          & 0.95 (0.02)          & 1.00 (0.03)          & 0.92 (0.05)          \\
airfoil (1503)           & \textbf{1.23 (0.10)} & \textbf{1.21 (0.10)}     & 1.06 (0.07)          & \textbf{1.09 (0.11)} & 1.05 (0.04)          & 1.07 (0.04)          & \textbf{1.10 (0.07)} \\
winewhite (1599)        & \textbf{1.11 (0.03)}  & \textbf{1.11 (0.03)}     & 0.96 (0.02)          & 0.95 (0.02)          & 1.03 (0.02)           & 1.04 (0.01)          & 1.02 (0.03)          \\
star  (2161)              & \textbf{0.99 (0.02)} & \textbf{0.99 (0.01)}     & 0.90 (0.02)          & 0.86 (0.03)          & \textbf{1.00 (0.01)}  & \textbf{1.00 (0.01)} & 0.91 (0.02)           \\
winered  (4898)        & \textbf{1.24 (0.06)} & \textbf{1.25 (0.06)}     & 0.88 (0.03)          & 0.83 (0.04)          & 1.01 (0.02)          & 1.01 (0.02)          & 1.02 (0.05)          \\
cycle  (9568)    & 1.00 (0.02)          & \textbf{1.05 (0.02)}     & 0.88 (0.02)          & 0.87 (0.02)          & 0.89 (0.02)          & 0.95 (0.02)          & 0.89 (0.02)           \\
electric   (10000)       & 0.96 (0.02)          & 0.96 (0.02)              & \textbf{1.00 (0.02)} & \textbf{1.00 (0.03)}  & \textbf{1.01 (0.01)} & \textbf{1.02 (0.01)}  & \textbf{1.00 (0.03)} \\
meps19 (15781)           & \textbf{1.12 (0.07)} & \textbf{1.11 (0.08)}     & \textbf{1.13 (0.06)} & \textbf{1.14 (0.06)} & \textbf{1.13 (0.05)}  & \textbf{1.11 (0.05)} & \textbf{1.13 (0.06)} \\
superconductivity  (21263)  & \textbf{1.10 (0.03)} & 1.09 (0.03)              & \textbf{1.14 (0.03)} & \textbf{1.15 (0.03)} & \textbf{1.12 (0.03)} & \textbf{1.14 (0.03)} & \textbf{1.15 (0.03)} \\
homes   (21613)            & \textbf{1.64 (0.08)} & \textbf{1.64 (0.08)}     & \textbf{1.36 (0.04)}  & \textbf{1.40 (0.05)} & \textbf{1.36 (0.04)} & \textbf{1.44 (0.05)} & \textbf{1.50 (0.06)} \\
protein  (45730)         & \textbf{1.02 (0.01)} & \textbf{1.03 (0.01)}     & 0.99 (0.01)          & 0.99 (0.01)          & 0.99 (0.01)          & 1.00 (0.01)            & 1.00 (0.01)          \\
WEC    (54000)           & \textbf{1.25 (0.03)} & \textbf{1.25 (0.03)}     & 1.16 (0.02)            & 1.16 (0.02)            & 1.16 (0.02)            & 1.16 (0.02)          & 1.21 (0.02)          \\ \hline
\end{tabular}
\end{adjustbox}
\end{table*}

These findings directly complement the ILR results from the previous section, providing a mechanistic explanation for why \ourmethod{} produces wider intervals for atypical points. By confirming that the credal envelope selectively expands in regions of low covariate support, these results indicate that our decomposition successfully isolates local epistemic uncertainty from global aleatoric noise.

\section{Final Remarks}

We introduced \ourmethod{}, a regression conformal predictor that
makes \emph{local epistemic uncertainty} explicit through covariate-dependent
credal envelopes, while retaining split conformal’s finite-sample,
distribution-free marginal coverage. By separating roles---credal modeling for
epistemic structure and conformal calibration for validity---the resulting
intervals are interpretable and admit a practical decomposition into aleatoric,
epistemic, and calibration components, making it easier to diagnose \emph{why} uncertainty
is large at a given $x$.

The connections made in this paper open several avenues. On
the modeling side, $\mathcal F_0(x)$ could be built from predictive distributions
induced by multiple priors or hyperpriors, alternative likelihood families, or
ensembles, thereby representing model ambiguity beyond standard parameter
uncertainty. 
For instance, in a Gaussian
process, one could define a local credal set by taking the union of posterior
predictive distributions obtained under different kernels and bandwidths. 
On the set-construction side, one can replace central-quantile
envelopes with HPD/IHDR-style regions or other credal summaries better suited
to skewed or multimodal conditionals, and then conformalize via an appropriate
distance-to-region score. From a scalability perspective, deep kernel learning
and learned embeddings $\phi$ offer a path to faster and more faithful
data-support measures for $\gamma(x)$. Finally, the same credal-conformal recipe should extend naturally
to classification and structured outputs, where epistemic effects are often
even more pronounced, and to decision-focused variants that tune the credal or
conformal components for downstream utility subject to coverage constraints. Code implementing \ourmethod{} and reproducing the experiments is available at \url{https://github.com/Monoxido45/CREDO}


\section*{Acknowledgements} %

  S. J. S. was supported by the UKRI Turing AI World-Leading Researcher Fellowship, [EP/W002973/1].
  R. I. is grateful for the financial support of CNPq (422705/2021-7, 305065/2023-8 and 403458/2025-0) and FAPESP (grant 2023/07068-1).
  L. M. C. C. is grateful for the financial support of FAPESP (grants 2022/08579-7 and 2025/06168-8).
  F. W. was supported by the Engineering and Physical Sciences Research Council [EP/Y030826/1]. 
  B. M. R. is grateful for the financial support of FAPESP (grant 2025/04853-5).
  
\newpage

\bibliography{bib2}

@article{cabezas2025cp4sbi,
  title={CP4SBI: Local Conformal Calibration of Credible Sets in Simulation-Based Inference},
  author={Cabezas, Luben and Santos, Vagner S and Ramos, Thiago R and Rodrigues, Pedro LC and Izbicki, Rafael},
  journal={arXiv preprint arXiv:2508.17077},
  year={2025}
}

@book{levi2,
	author = {Isaac Levi},
	publisher = {London, UK : MIT Press},
	title = {The Enterprise of Knowledge},
	year = {1980}}

@ARTICLE{abellan,
    author = {Joaquin Abellan and George Jiri Klir and Serafin Moral},
    title = {Disaggregated total uncertainty measure for credal sets},
    journal = {International Journal of General Systems},
    year = {2006},
    number = {35},
	volume = {1},
    pages = {29--44}
}

@inproceedings{breunig2000lof,
  title={LOF: identifying density-based local outliers},
  author={Breunig, Markus M and Kriegel, Hans-Peter and Ng, Raymond T and Sander, J{\"o}rg},
  booktitle={Proceedings of the 2000 ACM SIGMOD international conference on Management of data},
  pages={93--104},
  year={2000}
}

@InProceedings{gal2016dropout,
	title = {{Dropout as a bayesian approximation: Representing model uncertainty in deep learning}},
	author = {Gal, Yarin and Ghahramani, Zoubin},
	booktitle = {International Conference on Machine Learning},
	year = {2016},
	series = {Proceedings of Machine Learning Research},
	pages = {1050-1059},
	publisher = {PMLR}
}

@article{dorogush2018catboost,
  title={CatBoost: gradient boosting with categorical features support},
  author={Dorogush, Anna Veronika and Ershov, Vasily and Gulin, Andrey},
  journal={arXiv preprint arXiv:1810.11363},
  year={2018}
}

@article{sesia2020comparison,
  title={A comparison of some conformal quantile regression methods},
  author={Sesia, Matteo and Cand{\`e}s, Emmanuel J},
  journal={Stat},
  volume={9},
  number={1},
  pages={e261},
  year={2020},
  publisher={Wiley Online Library}
}

@article{van2008visualizing,
  title={Visualizing data using t-SNE.},
  author={Van der Maaten, Laurens and Hinton, Geoffrey},
  journal={Journal of machine learning research},
  volume={9},
  number={11},
  year={2008}
}

@inproceedings{caprio2023constriction,
  title={Constriction for sets of probabilities},
  author={Caprio, Michele and Seidenfeld, Teddy},
  booktitle={International Symposium on Imprecise Probability: Theories and Applications},
  pages={84--95},
  year={2023},
  organization={PMLR}
}

@article{dilation,
author = {Teddy Seidenfeld and Larry Wasserman},
title = {{Dilation for Sets of Probabilities}},
volume = {21},
journal = {The Annals of Statistics},
number = {3},
publisher = {Institute of Mathematical Statistics},
pages = {1139 -- 1154},
year = {1993},
doi = {10.1214/aos/1176349254},
URL = {https://doi.org/10.1214/aos/1176349254}
}

@misc{caprio2025joyscategoricalconformalprediction,
      title={The Joys of Categorical Conformal Prediction}, 
      author={Michele Caprio},
      year={2025},
      eprint={2507.04441},
      archivePrefix={arXiv},
      primaryClass={stat.ML},
      url={https://arxiv.org/abs/2507.04441}, 
}

@book{TroffaesDeCooman2014,
  author    = {Troffaes, Matthias C. M. and de Cooman, Gert},
  title     = {Lower Previsions},
  year      = {2014},
  series    = {Wiley Series in Probability and Statistics},
  publisher = {John Wiley \& Sons},
  address   = {Chichester, United Kingdom},
  isbn      = {978-0-470-72377-7},          
  doi       = {10.1002/9781118762622},     
  url       = {https://doi.org/10.1002/9781118762622}
}

@book{augustin2014introduction,
  title={Introduction to imprecise probabilities},
  author={Augustin, Thomas and Coolen, Frank PA and De Cooman, Gert and Troffaes, Matthias CM},
  volume={591},
  year={2014},
  publisher={John Wiley \& Sons}
}

@book{vovk2022algorithmic,
  title={Algorithmic Learning in a Random World},
  author={Vovk, Vladimir and Gammerman, Alexander and Shafer, Glenn},
  edition   = {2nd},
  year={2022},
  publisher={Springer Nature}
}

@misc{AngelopoulosBates2021,
  title        = {A gentle introduction to conformal prediction and distribution-free uncertainty quantification},
  author       = {Angelopoulos, Anastasios N. and Bates, Stephen},
  howpublished = {arXiv preprint arXiv:2107.07511},
  year         = {2021}
}

@inproceedings{IzbickiShimizuStern2020,
  title     = {Flexible distribution-free conditional predictive bands using density estimators},
  author    = {Izbicki, Rafael and Shimizu, Gilson and Stern, Rafael},
  booktitle = {International Conference on Artificial Intelligence and Statistics},
  pages     = {3068--3077},
  publisher = {PMLR},
  year      = {2020}
}

@article{IzbickiShimizuStern2022,
  title   = {Cd-split and hpd-split: Efficient conformal regions in high dimensions},
  author  = {Izbicki, Rafael and Shimizu, Gilson and Stern, Rafael B.},
  journal = {Journal of Machine Learning Research},
  volume  = {23},
  number  = {87},
  pages   = {1--32},
  year    = {2022}
}

@article{CabezasOttoIzbickiStern2025,
  title   = {Regression trees for fast and adaptive prediction intervals},
  author  = {Cabezas, Luben M. C. and Otto, Mateus P. and Izbicki, Rafael and Stern, Rafael B.},
  journal = {Information Sciences},
  volume  = {686},
  pages   = {121369},
  year    = {2025}
}

@inproceedings{FongHolmes2021,
  title     = {Conformal Bayesian computation},
  author    = {Fong, Edwin and Holmes, Chris C.},
  booktitle = {Advances in Neural Information Processing Systems},
  volume    = {34},
  pages     = {18268--18279},
  year      = {2021}
}

@book{Izbicki2025,
  author    = {Rafael Izbicki},
  title     = {Machine Learning Beyond Point Predictions: Uncertainty Quantification},
  edition   = {1st},
  year      = {2025},
  pages     = {260},
  isbn      = {978-65-01-20272-3}
}

@inproceedings{wilson2016deep,
  title={Deep kernel learning},
  author={Wilson, Andrew Gordon and Hu, Zhiting and Salakhutdinov, Ruslan and Xing, Eric P},
  booktitle={Artificial intelligence and statistics},
  pages={370--378},
  year={2016},
  organization={PMLR}
}

@misc{CabezasEtAlEPICSCORE,
  title   = {Epistemic Uncertainty in Conformal Scores: A Unified Approach},
  author  = {Cabezas, Luben M. C. and Santos, Vagner S. and Ramos, Thiago R. and Izbicki, Rafael},
  note    = {Manuscript (UAI submission)},
  year    = {2025}
}

@inproceedings{Papadopoulos2002,
  title     = {Inductive Confidence Machines for Regression},
  author    = {Papadopoulos, Harris and Proedrou, Kostas and Vovk, Volodya and Gammerman, Alex},
  booktitle = {Machine Learning: ECML 2002},
  pages     = {345--356},
  year      = {2002},
  publisher = {Springer}
}

@article{HullermierWaegeman2021,
  title   = {Aleatoric and Epistemic Uncertainty in Machine Learning: An Introduction to Concepts and Methods},
  author  = {H{\"u}llermeier, Eyke and Waegeman, Willem},
  journal = {Machine Learning},
  volume  = {110},
  pages   = {457--506},
  year    = {2021}
}

@article{GneitingRaftery2007,
  title={Strictly proper scoring rules, prediction, and estimation},
  author={Gneiting, Tilmann and Raftery, Adrian E},
  journal={Journal of the American statistical Association},
  volume={102},
  number={477},
  pages={359--378},
  year={2007},
  publisher={Taylor \& Francis}
}

@inproceedings{Rossellini2024,
  title     = {Integrating Uncertainty Awareness into Conformalized Quantile Regression},
  author    = {Rossellini, Raphael and Foygel Barber, Rina and Willett, Rebecca},
  booktitle = {International Conference on Artificial Intelligence and Statistics},
  pages     = {1540--1548},
  year      = {2024},
  publisher = {PMLR}
}

@misc{Jaber2024,
  title  = {Conformal Approach to Gaussian Process Surrogate Evaluation with Coverage Guarantees},
  author = {Jaber, Edgar and Blot, Vincent and Brunel, Nicolas and Chabridon, Vincent and Remy, Emmanuel and Iooss, Bertrand and Lucor, Didier and Mougeot, Mathilde and Leite, Alessandro},
  note   = {arXiv:2401.07733},
  year   = {2024}
}

@inproceedings{PionVazquez2024,
  title     = {Gaussian Process Interpolation with Conformal Prediction: Methods and Comparative Analysis},
  author    = {Pion, Aur{\'e}lien and Vazquez, Emmanuel},
  booktitle = {LOD 2024: International Conference on Machine Learning, Optimization, and Data Science},
  year      = {2024}
}

@misc{Cocheteux2025,
  title  = {Uncertainty-Aware Online Extrinsic Calibration: A Conformal Prediction Approach},
  author = {Cocheteux, Mathieu and Moreau, Julien and Davoine, Franck},
  note   = {arXiv:2501.06878},
  year   = {2025}
}

@book{Vovk2005,
  title     = {Algorithmic Learning in a Random World},
  author    = {Vovk, Vladimir and Gammerman, Alex and Shafer, Glenn},
  publisher = {Springer},
  year      = {2005}
}

@article{ShaferVovk2008,
  title   = {A Tutorial on Conformal Prediction},
  author  = {Shafer, Glenn and Vovk, Vladimir},
  journal = {Journal of Machine Learning Research},
  volume  = {9},
  pages   = {371--421},
  year    = {2008}
}

@article{Chipman2010BART,
  author    = {Hugh A. Chipman and Edward I. George and Robert E. McCulloch},
  title     = {BART: Bayesian Additive Regression Trees},
  journal   = {The Annals of Applied Statistics},
  volume    = {4},
  number    = {1},
  pages     = {266--298},
  year      = {2010},
  publisher = {Institute of Mathematical Statistics},
  doi       = {10.1214/09-AOAS285}
}

@misc{Wang2022LocalGPBART,
  author       = {Meijiang Wang and Jingyu He and P. Richard Hahn},
  title        = {Local Gaussian process extrapolation for BART models with applications to causal inference},
  howpublished = {arXiv preprint arXiv:2204.10963},
  year         = {2022},
  note         = {Version 2 (revised Feb 24, 2023)},
  doi          = {10.48550/arXiv.2204.10963},
  eprint       = {2204.10963},
  archivePrefix= {arXiv},
  primaryClass = {stat.ME}
}

@article{Lei2018,
  title={Distribution-Free Predictive Inference for Regression},
  author={Lei, Jing and G'Sell, Max and Rinaldo, Alessandro and Tibshirani, Ryan J. and Wasserman, Larry},
  journal={Journal of the American Statistical Association},
  year={2018},
  volume={113},
  number={523},
  pages={1094--1111}
}

@article{Romano2019,
  title={Conformalized Quantile Regression},
  author={Romano, Yaniv and Patterson, Evan and Cand{\`e}s, Emmanuel J.},
  journal={Advances in Neural Information Processing Systems},
  year={2019}
}

@book{Walley1991,
  title={Statistical Reasoning with Imprecise Probabilities},
  author={Walley, Peter},
  year={1991},
  publisher={Chapman and Hall/CRC}
}

@mastersthesis{jansen2013,
  title={Robust Bayesian inference under model misspecification},
  author={Jansen, Laurens},
  school={Leiden University},
  year={2013},
  note={Chapter 4.5}
}

@book{bernardo2009bayesian,
  title={Bayesian theory},
  author={Bernardo, Jose M and Smith, Adrian FM},
  volume={405},
  year={2009},
  publisher={John Wiley \& Sons}
}

@misc{wu2026BCP,
      title={Bayesian Conformal Prediction as a Decision Risk Problem}, 
      author={Fanyi Wu and Veronika Lohmanova and Samuel Kaski and Michele Caprio},
      year={2026},
      eprint={2602.03331},
      archivePrefix={arXiv},
      primaryClass={cs.LG},
      url={https://arxiv.org/abs/2602.03331}, 
}

@article{bariletto2025conformalized,
  title        = {Conformalized Bayesian Inference, with Applications to Random Partition Models},
  author       = {Bariletto, Nicola and Ho, Nhat and Rinaldo, Alessandro},
  journal      = {arXiv preprint arXiv:2511.05746v1},
  year         = {2025},
  url          = {https://arxiv.org/abs/2511.05746v1}
}

@article{zhao2022analysis,
  title={Analysis of knn density estimation},
  author={Zhao, Puning and Lai, Lifeng},
  journal={IEEE Transactions on Information Theory},
  volume={68},
  number={12},
  pages={7971--7995},
  year={2022},
  publisher={IEEE}
}

@misc{ovadia2019,
      title={Can You Trust Your Model's Uncertainty? Evaluating Predictive Uncertainty Under Dataset Shift}, 
      author={Yaniv Ovadia and Emily Fertig and Jie Ren and Zachary Nado and D Sculley and Sebastian Nowozin and Joshua V. Dillon and Balaji Lakshminarayanan and Jasper Snoek},
      year={2019},
      eprint={1906.02530},
      archivePrefix={arXiv},
      primaryClass={stat.ML},
      url={https://arxiv.org/abs/1906.02530}, 
}

@article{Cella2022,
   title={Validity, consonant plausibility measures, and conformal prediction},
   volume={141},
   ISSN={0888-613X},
   url={http://dx.doi.org/10.1016/j.ijar.2021.07.013},
   DOI={10.1016/j.ijar.2021.07.013},
   journal={International Journal of Approximate Reasoning},
   publisher={Elsevier BV},
   author={Cella, Leonardo and Martin, Ryan},
   year={2022},
   month=feb, 
   pages={110–130} 
}

@misc{caprio2025conformalpredictionregionsimprecise,
      title={Conformal Prediction Regions are Imprecise Highest Density Regions}, 
      author={Michele Caprio and Yusuf Sale and Eyke H{\"u}llermeier},
      year={2025},
      eprint={2502.06331},
      archivePrefix={arXiv},
      primaryClass={stat.ML},
      url={https://arxiv.org/abs/2502.06331}, 
}

@misc{caprio2025conformalizedcredalregionsclassification,
      title={Conformalized Credal Regions for Classification with Ambiguous Ground Truth}, 
      author={Michele Caprio and David Stutz and Shuo Li and Arnaud Doucet},
      year={2025},
      eprint={2411.04852},
      archivePrefix={arXiv},
      primaryClass={stat.ML},
      url={https://arxiv.org/abs/2411.04852}, 
}

@inproceedings{Alireza2024,
 author = {Javanmardi, Alireza and Stutz, David and H\"{u}llermeier, Eyke},
 booktitle = {Advances in Neural Information Processing Systems},
 doi = {10.52202/079017-3714},
 editor = {A. Globerson and L. Mackey and D. Belgrave and A. Fan and U. Paquet and J. Tomczak and C. Zhang},
 pages = {116987--117014},
 publisher = {Curran Associates, Inc.},
 title = {Conformalized Credal Set Predictors},
 url = {https://proceedings.neurips.cc/paper_files/paper/2024/file/d42a8bf2f40555d4a5120300f98c88f6-Paper-Conference.pdf},
 volume = {37},
 year = {2024}
}

@misc{Chau2026MMICP,
      title={Quantifying Epistemic Predictive Uncertainty in Conformal Prediction}, 
      author={Siu Lun Chau and Soroush H. Zargarbashi and Yusuf Sale and Michele Caprio},
      year={2026},
      eprint={2602.01667},
      archivePrefix={arXiv},
      primaryClass={cs.LG},
      url={https://arxiv.org/abs/2602.01667}, 
}

\newpage

\onecolumn

\textbf{\huge \center{Appendix}}

\appendix

\section{\ourmethod{} algorithms}
\label{sec:alg}
This section presents the algorithmic framework for \ourmethod{}. Algorithm \ref{alg:credo} details the construction and conformalization of the credal envelope, while Algorithm \ref{alg:adapt_gamma} specifies the adaptive $\gamma(x)$ scheme. Finally, Algorithm \ref{alg:unc_decomp} outlines the decomposition of the prediction intervals into aleatoric and epistemic components.
\begin{algorithm}[H]
\caption{\ourmethod{}: Conformalized credal regression intervals via endpoint-trimmed posterior envelopes}
\label{alg:credo}
\begin{algorithmic}[1]
\Require Exchangeable data $D=\{(X_i,Y_i)\}_{i=1}^n$; split sizes with $|D_{\mathrm{cal}}|=m$; query observation $x$.
\Require Levels $\alpha_0\in(0,1)$ (credal nominal) and $\alpha\in(0,1)$ (conformal).
\Require Conditional model $\{F_\theta(\cdot\mid x):\theta\in\Theta\}$, prior $\pi(\theta)$, posterior sampler; MC size $B$.
\Require Trimming: either constant $\gamma\in(0,1)$ or covariate-dependent $\gamma(x)$ using Algorithm \ref{alg:adapt_gamma}.
\Ensure For any query covariate $x$, output $C(x)=[\ell(x)-\hat\tau,\;u(x)+\hat\tau]$ and (optionally) the uncertainty decomposition.

\vspace{0.2em}
\State \textbf{Split:} Randomly split $D$ into $D_{\mathrm{tr}}$ and $D_{\mathrm{cal}}$ (size $m$).
\State \textbf{Posterior:} Fit on $D_{\mathrm{tr}}$ to obtain $\pi(\theta\mid D_{\mathrm{tr}})$; draw $\theta^{(1)},\ldots,\theta^{(B)}\sim \pi(\theta\mid D_{\mathrm{tr}})$.

\vspace{0.2em}
\Function{Envelope}{$x$}
    \For{$b=1$ to $B$}
        \State $q_L^{(b)}(x)\gets F_{\theta^{(b)}}^{-1}(\alpha_0/2\mid x)$;\quad
               $q_U^{(b)}(x)\gets F_{\theta^{(b)}}^{-1}(1-\alpha_0/2\mid x)$.
    \EndFor
    \If{Constant $\gamma$}
    \State $\gamma(x) \gets \gamma$
    \EndIf
    \State $C_L(x)\gets \operatorname{Quantile}_{\gamma(x)/2}\big(\{q_L^{(b)}(x)\}_{b=1}^B\big)$; $C_U(x)\gets \operatorname{Quantile}_{1-\gamma(x)/2}\big(\{q_U^{(b)}(x)\}_{b=1}^B\big)$.
    \State $\ell(x)\gets C_L(x)$; $u(x)\gets C_U(x)$ \Comment{trimmed endpoint envelope \eqref{eq:trimmed-envelope}}
    \State \Return $(\ell(x),u(x),\{q_L^{(b)}(x),q_U^{(b)}(x)\}_{b=1}^B)$.
\EndFunction

\vspace{0.2em}
\State \textbf{Conformal calibration on $D_{\mathrm{cal}}$:}
\ForAll{$(x_i,y_i)\in D_{\mathrm{cal}}$}
    \State $(\ell_i,u_i,\_)\gets \Call{Envelope}{x_i}$.
    \State $s_i\gets \max\{\ell_i-y_i,\;y_i-u_i\}$ \Comment{distance-to-envelope score \eqref{eq:score}}
\EndFor
\State $\hat\tau\gets$ the $\lceil (m+1)(1-\alpha)\rceil$-th order statistic of $\{s_i\}_{i=1}^m$.

\vspace{0.2em}
\State \textbf{Prediction for a new covariate $x$:}
\State $(\ell(x),u(x),\{q_L^{(b)}(x),q_U^{(b)}(x)\})\gets \Call{Envelope}{x}$.
\If{requires uncertainty decomposition at $x$}
\State Run Algorithm \ref{alg:unc_decomp} with input $((\ell(x),u(x)),\{q_L^{(b)}
(x),q_U^{(b)}(x)\}_{b = 1}^B)$
\State Returns $\widehat U_A(x), \widehat U_E(x)$
\EndIf
\State \textbf{Output} $C(x)=[\ell(x)-\hat\tau,\;u(x)+\hat\tau]$ and $\widehat U_A(x), \widehat U_E(x)$ if demanded \Comment{final set \eqref{eq:final-set}}
\end{algorithmic}
\end{algorithm}

\begin{algorithm}[H]
\caption{Adaptive $\gamma$}
\label{alg:adapt_gamma}
\begin{algorithmic}[1]
\vspace{0.2em}
\Require Training set $D_{tr}$ and observation $x$.
\Require Representation $\phi(\cdot)$, $k$, constants $\gamma_{\min},\gamma_{\max},\tau_\gamma,m_\gamma$, and $\varepsilon>0$.
\Ensure For any query covariate $x$, output $\gamma(x)$.
\vspace{0.2em}
\State For each training covariate $X_i$ in $D_{\mathrm{tr}}$, compute $r_k(X_i):=\|\phi(X_i)-\phi(X_{(k)}(X_i))\|$.
\State Set reference quantiles $q_{\mathrm{lo}}=\operatorname{Quantile}_{0.50}(\{r_k(X_i)\})$ and $q_{\mathrm{hi}}=\operatorname{Quantile}_{0.95}(\{r_k(X_i)\})$.
\State Find $X_{(k)}(x)$ as the $k$-th nearest neighbor of $x$ among $\{X_i:(X_i,Y_i)\in D_{\mathrm{tr}}\}$.

\vspace{0.2em}
\State $r_k(x)\gets \|\phi(x)-\phi(X_{(k)}(x))\|$.
\State $\mathrm{sc}(x)\gets \dfrac{r_k(x)-q_{\mathrm{lo}}}{q_{\mathrm{hi}}-q_{\mathrm{lo}}+\varepsilon}$. \Comment{Defining scarcity score}
\State \textbf{Output} $\gamma(x)\gets \gamma_{\max}-(\gamma_{\max}-\gamma_{\min})\,\sigma\!\left(\dfrac{\mathrm{sc}(x)-m_\gamma}{\tau_\gamma}\right)$.  \Comment{Use scarcity score to compute $\gamma$}
\end{algorithmic}
\end{algorithm}

\begin{algorithm}[H]
\caption{Uncertainty Decomposition}
\label{alg:unc_decomp}
\begin{algorithmic}[1]
\vspace{0.2em}
\Require Lower and upper envelope bounds at observation $x$, $(\ell(x), u(x))$.
\Require Lower and upper quantile vectors at observation $x$, $\{q_L^{(b)}(x),q_U^{(b)}(x)\}_{b=1}^B)$.
\Ensure For any query covariate $x$, output $(\widehat{U}_A(x), \widehat{U}_E(x))$.
\State \textbf{Optional uncertainty decomposition at $x$:}
\State $\widehat U_A(x)\gets \frac{1}{B}\sum_{b=1}^B \big(q_U^{(b)}(x)-q_L^{(b)}(x)\big)$.
\State $\widehat U_E(x)\gets (u(x)-\ell(x))-\widehat U_A(x)$;\quad conformal term $=2\hat\tau$.
\State (So $|C(x)|=\widehat U_A(x)+\widehat U_E(x)+2\hat\tau$).
\State \textbf{Output} $\widehat{U}_A(x), \widehat{U}_E(x)$.
\end{algorithmic}
\end{algorithm}

\section{Additional results and experiments details}
In this section, we provide a comprehensive overview of our evaluation metrics, baseline architectures, and extended comparative results. We also detail the outlier detection methodology employed and present uncertainty decomposition analyses for both simulated and real-world datasets. Table~\ref{tab:realdata_new} summarizes the regression benchmarks used throughout our study; these datasets are standard in the conformal prediction literature \citep{Romano2019, Rossellini2024} and offer a diverse range of sample sizes and dimensionalities. Together, they provide a robust testbed for evaluating the performance, adaptivity, and interpretability of the proposed approach.
\label{appendix:additional_res}
\begin{table}[ht]
\caption{Benchmark Datasets Overview. We summarize the sample size ($n$), feature count ($p$), and the original scientific domain of each benchmark.}
\label{tab:realdata_new}
\centering
\small
\begin{tabular}{@{}lrrrl@{}}
\toprule
\textbf{Dataset} & \textbf{\textit{n}} & \textbf{\textit{p}} & \textbf{Scientific Domain} & \textbf{Data Source} \\ \midrule
Concrete & 1,030 & 8 & Civil Engineering & \href{https://archive.ics.uci.edu/dataset/165/concrete+compressive+strength}{UCI Repository} \\
Airfoil & 1,503 & 5 & Aeroacoustics & \href{https://archive.ics.uci.edu/dataset/291/airfoil+self+noise}{UCI Repository} \\
Wine White & 1,599 & 11 & Enology & \href{https://archive.ics.uci.edu/dataset/186/wine+quality}{UCI Repository} \\
Star & 2,161 & 48 & Pedagogy & \href{https://dataverse.harvard.edu/dataset.xhtml?persistentId=doi:10.7910/DVN/SIWH9F}{Harvard Dataverse} \\
Wine Red & 4,898 & 11 & Enology & \href{https://archive.ics.uci.edu/dataset/186/wine+quality}{UCI Repository} \\
Cycle & 9,568 & 4 & Thermodynamics & \href{http://archive.ics.uci.edu/dataset/294/combined+cycle+power+plant}{UCI Repository} \\
Electric & 10,000 & 12 & Power Systems & \href{http://archive.ics.uci.edu/ml/datasets/Electrical+Grid+Stability+Simulated+Data+}{UCI Repository} \\
Meps19 & 15,781 & 141 & Health Economics & \href{https://meps.ahrq.gov/mepsweb/data_stats/download_data_files_detail.jsp?cboPufNumber=HC-181}{AHRQ (MEPS)} \\
Superconductivity & 21,263 & 81 & Condensed Matter Physics & \href{http://archive.ics.uci.edu/ml/datasets/Superconductivty+Data}{UCI Repository} \\
Homes & 21,613 & 17 & Econometrics & \href{https://www.kaggle.com/datasets/harlfoxem/housesalesprediction}{Kaggle} \\
Protein & 45,730 & 8 & Structural Bioinformatics & \href{http://archive.ics.uci.edu/dataset/265/physicochemical+properties+of+protein+tertiary+structure}{UCI Repository} \\
Wave Energy (WEC) & 54,000 & 49 & Renewable Energy & \href{https://archive.ics.uci.edu/dataset/882/large-scale+wave+energy+farm}{UCI Repository} \\ \bottomrule
\end{tabular}
\end{table}

\subsection{Outlier detection method}
\label{appendix:outlier_detection}

To assess how prediction sets adapt to data scarcity, we categorize test set observations using a model-agnostic density approach. We first project the feature space into two dimensions using t-SNE \citep{van2008visualizing} and then apply the Local Outlier Factor (LOF) method \citep{breunig2000lof}. Using $k=15$ nearest neighbors and a $5\%$ contamination rate, we identify a set of local anomalies to serve as our outlier group. 

For a robust comparative analysis, we contrast these outliers against two specific inlier reference groups:
\begin{itemize}
    \item \textbf{Interval Ratio (ILR)}: We compare outliers against the $20\%$ most central inliers (those with the lowest LOF scores) to assess global adaptivity across all benchmarks.
    \item \textbf{Uncertainty Decomposition}: We narrow the focus to the $10\%$ most central inliers to provide a more rigorous baseline for validating the disentanglement of epistemic and aleatoric signals.
\end{itemize}


The LOF score provides a structured, model-agnostic characterization of data density, serving as an external proxy for covariate support. It is important to distinguish this evaluation procedure from the adaptive $\gamma(x)$ mechanism within \ourmethod{}. While $\gamma(x)$ is an intrinsic component used to manage model ambiguity during inference, LOF is employed here strictly as an independent diagnostic tool for categorization in the test set. This ensures that our assessment of epistemic uncertainty—specifically the expectation that outliers in sparse regions yield wider sets while inliers in dense regions yield narrower ones—is verified against a standard external baseline, rather than being an artifact of the model’s internal representations.

\subsection{Evaluation metrics details}
\label{appendix:evaluation_metrics}

Let $\hat{C}(\mathbf{X})$ denote a generic prediction set. Given a test set $\{(\mathbf{X}_i, Y_i)\}_{i=1}^m$, we evaluate the performance of all methods using three primary categories of metrics: reliability, efficiency, and epistemic awareness.

\paragraph{1. Reliability and Coverage Accuracy}
To verify that the nominal coverage $1-\alpha$ is maintained both globally and in data-sparse regions, we report:
\begin{itemize}
    \item \textbf{Average Marginal Coverage (AMC)}: The empirical proportion of test points covered by the prediction sets:
    \begin{equation*}
        \text{AMC} = \frac{1}{m} \sum_{i=1}^{m} \mathbb{I} \left( Y_i \in \hat{C}(\mathbf{X}_i) \right).
    \end{equation*}
    
    \item \textbf{Average Coverage on Outliers (ACO)}: The empirical coverage calculated exclusively over the set of detected outliers $I_{\text{out}}$ (see Section \ref{appendix:outlier_detection}):
    \begin{equation*}
        \text{ACO} = \frac{1}{|I_{\text{out}}|} \sum_{i \in I_{\text{out}}} \mathbb{I} \left( Y_i \in \hat{C}(\mathbf{X}_i) \right).
    \end{equation*}
    Maintaining ACO near the nominal level $1-\alpha$ indicates that the method remains reliable even under significant sparsity.
\end{itemize}

\paragraph{2. Conditional Coverage Efficiency}
To evaluate the quality of the intervals beyond marginal coverage, we employ proxies for conditional validity:
\begin{itemize}
    \item \textbf{Scaled Mean Interval Score (SMIS)}: Introduced by \citet{GneitingRaftery2007}, this score evaluates the trade-off between interval width and miscoverage:
    \begin{align*}
        \text{SMIS} = \frac{1}{m} \sum_{i=1}^{m} \Bigg[ &\left( \max \hat{C}(\X_i) - \min \hat{C}(\X_i)\right) \\
        &+ \frac{2}{\alpha} \cdot \left(\min \hat{C}(\X_i) - Y_i \right) \cdot \I \left\{Y_i < \min \hat{C}(\X_i) \right\} \\
        &+ \frac{2}{\alpha} \cdot \left(Y_i - \max \hat{C}(\X_i) \right)\cdot \I \left\{Y_i > \max \hat{C}(\X_i \right\} \Bigg] \; ,
    \end{align*}
    where $\min{\hat{C}(\X)}$ and $\max{\hat{C}(\X)}$ are the interval endpoints. By penalizing both excessive width and miscoverage, SMIS identifies methods that achieve high efficiency per unit of coverage, which translates into more locally-adaptive methods. Since this proper scoring rule is minimized when the interval boundaries match the true conditional quantiles, it serves as a proxy for \textit{conditional coverage} adaptivity.
\end{itemize}
    

\paragraph{3. Epistemic Awareness}
\begin{itemize}
    \item \textbf{Outlier-to-Inlier Length Ratio (ILR)} \citep{CabezasEtAlEPICSCORE}: This metric assesses how the method distinguishes between data-rich and data-sparse regions. Let $\bar{W}(I) = \frac{1}{|I|} \sum_{i \in I} (\max \hat{C}(\X_i) - \min \hat{C}(\X_i))$ be the average width over an index set $I$. The ILR is defined as:
    \begin{equation*}
        \text{ILR} = \frac{\bar{W}(I_{\text{out}})}{\bar{W}(I_{\text{in}})},
    \end{equation*}
    where $I_{\text{out}}$ contains all detected outliers (5\% contamination) and $I_{\text{in}}$ contains the $20\%$ most central inliers. A higher ILR indicates successful identification of epistemic uncertainty through selective interval widening in sparse regions.
\end{itemize}

\subsection{Baselines details}
\label{appendix:baseline_details}

\subsubsection*{Conformal Baselines and Architectures}
We evaluate \ourmethod{} against a range of state-of-the-art conformal regressors. All methods utilize CatBoost \citep{dorogush2018catboost} as the underlying quantile regression model. The baselines are categorized as follows:

\begin{itemize}

\item \textbf{CQR \& CQR-r} \citep{Romano2019, sesia2020comparison}: These represent the standard for conformal quantile regression. While CQR provides a global calibration of quantile estimates, CQR-r incorporates a randomized scaling factor to enhance local adaptivity. Both methods focus primarily on capturing aleatoric uncertainty through the base quantile regressor.

\item \textbf{UACQR-S \& UACQR-P} \citep{Rossellini2024}: These variants integrate epistemic uncertainty by constructing an ensemble of $B=1000$ base models. \textbf{UACQR-S} (Standardized) inflates the conformal score using the ensemble standard deviation, effectively widening intervals in regions of high model disagreement. \textbf{UACQR-P} (Percentile) utilizes order statistics from the ensemble to derive a more conservative, disagreement-aware conformal set.

\item \textbf{EPICScore (MDN)} \citep{CabezasEtAlEPICSCORE}: This method integrates epistemic uncertainty through score rectification, leveraging the predictive conditional distribution of the conformity scores. Instead of using the raw conformal scores, EPICScore defines a new transformed score based on the predictive posterior distribution of the original scores. We utilize the \textbf{Mixture Density Network (MDN)} variant, which models the score distribution as a mixture of Gaussians and employs MC-Dropout \citep{gal2016dropout} to capture model ambiguity and estimate the posterior predictive distribution.

\end{itemize}

\subsubsection*{Implementation Details and Hyperparameters}

The experimental configurations for the baselines are detailed below:

\begin{itemize}

\item \textbf{CatBoost Base Model}: Following standard practice and ensuring a robust model, we fix the tree depth to $6$, $L_2$ leaf regularization to $3$, and the bagging temperature to $1$. The models are trained for $1000$ iterations with an auto-tuned learning rate.
\item \textbf{EPICScore MDN}: For the EPICScore baseline, we implement a Mixture Density Network (MDN) trained specifically to model the distribution of the conformity scores generated by the CatBoost base model. The MDN architecture consists of two hidden layers with $64$ nodes each and approximates the predictive distribution of these scores as a mixture of $5$ Gaussian components. To ensure convergence while preventing overfitting, the MDN is trained for a maximum of $2000$ epochs with an early stopping patience of $50$ iterations. We utilize a learning rate of $0.001$ and a dropout rate of $0.5$ to facilitate MC-dropout-based epistemic estimation. The training employs an adaptive batch size: $40$ for smaller samples, $130$ for larger sets such as \textit{Superconductivity}, and $250$ for the \textit{WEC} benchmark.

\item \textbf{Ensemble Settings}: The UACQR variants utilize the same ensemble size of catboost, $B=1000$, ensuring stable and robust estimation of the epistemic signals.
\end{itemize}

\subsection{Additional comparisons}
\label{appendix:additional_comparisons}
The additional marginal coverage and coverage on test set outliers are respectively displayed on Tables \ref{tab:marginal_coverage} and \ref{tab:coverage_outlier}. In general, all approaches respect the conformal marginal validity for every dataset in Table \ref{tab:marginal_coverage}, while \ourmethod{} stands out in the outlier coverage, achieving the closest coverage in the majority of scenarios, confirming that the widening of intervals in outliers also corresponds to an improvement in coverage in these areas.
\begin{table}[h]
\centering
\caption{\textbf{Marginal coverage on test data (target $1-\alpha$).}
Entries report the mean empirical coverage over 30 runs; parentheses show $2\times$ the standard deviation.
All methods achieve coverage close to the nominal level, with small deviations due to finite-sample variability and random splitting.}
\label{tab:marginal_coverage}
\begin{adjustbox}{max width=\textwidth}
\begin{tabular}{cccccccc}
\hline
\textbf{Dataset (n)}  & \textbf{CREDO} & \textbf{CREDO (adap.)} & \textbf{CQR}  & \textbf{CQR-r} & \textbf{UACQRS} & \textbf{UACQRP} & \textbf{EPIC} \\ \hline
concrete (1030)          & 0.91 (0.01) & 0.90 (0.01) & 0.90 (0.01) & 0.89 (0.01) & 0.90 (0.01) & 0.90 (0.01) & 0.90 (0.02) \\
airfoil (1503)           & 0.90 (0.01) & 0.91 (0.01) & 0.91 (0.01) & 0.89 (0.01) & 0.90 (0.01) & 0.90 (0.01) & 0.91 (0.01) \\
winewhite (1599)         & 0.90 (0.00) & 0.90 (0.01) & 0.91 (0.00) & 0.90 (0.00) & 0.90 (0.00) & 0.90 (0.00) & 0.90 (0.01) \\
star (2161)              & 0.90 (0.01) & 0.90 (0.01) & 0.90 (0.01) & 0.90 (0.01) & 0.90 (0.01) & 0.90 (0.01) & 0.91 (0.01) \\
winered (4898)           & 0.90 (0.01) & 0.90 (0.01) & 0.90 (0.01) & 0.90 (0.01) & 0.90 (0.01) & 0.90 (0.01) & 0.90 (0.01) \\
cycle (9568)             & 0.90 (0.00) & 0.90 (0.00) & 0.90 (0.00) & 0.90 (0.00) & 0.90 (0.00) & 0.90 (0.00) & 0.90 (0.01) \\
electric (10000)         & 0.90 (0.00) & 0.90 (0.00) & 0.90 (0.00) & 0.90 (0.00) & 0.90 (0.00) & 0.90 (0.00) & 0.90 (0.00) \\
meps19 (15781)           & 0.90 (0.00) & 0.90 (0.00) & 0.90 (0.00) & 0.90 (0.00) & 0.90 (0.00) & 0.92 (0.01) & 0.90 (0.00) \\
superconductivity (21263) & 0.90 (0.00) & 0.90 (0.00) & 0.90 (0.00) & 0.90 (0.00) & 0.90 (0.00) & 0.90 (0.00) & 0.90 (0.00) \\
protein (45730)          & 0.90 (0.00) & 0.90 (0.00) & 0.90 (0.00) & 0.90 (0.00) & 0.90 (0.00) & 0.90 (0.00) & 0.90 (0.00) \\
homes (21613)            & 0.90 (0.00) & 0.90 (0.00) & 0.90 (0.00) & 0.90 (0.00) & 0.90 (0.00) & 0.90 (0.00) & 0.90 (0.00) \\
WEC (54000)              & 0.90 (0.00) & 0.90 (0.00) & 0.90 (0.00) & 0.90 (0.00) & 0.90 (0.00) & 0.90 (0.00) & 0.90 (0.00) \\ \hline
\end{tabular}
\end{adjustbox}
\end{table}

\begin{table}[h]
\centering
\caption{\textbf{Coverage on test-set outliers (target $1-\alpha=0.90$).}
Outliers are defined using the criterion as in Appendix~\ref{appendix:outlier_detection}.
Entries report the mean outlier coverage over 30 runs; parentheses show $2\times$ the standard deviation.
Bold marks methods whose outlier coverage is closest to 0.90 within a 95\% confidence interval. Both \ourmethod{} versions have the closest coverage to the nominal level in 9 out of 12 datasets, showcasing its adaptability to outlier regions and its adequacy to deal with epistemic uncertainty.}
\label{tab:coverage_outlier}
\begin{adjustbox}{max width=\textwidth}
\begin{tabular}{ccccccccc}
\hline
\textbf{Dataset (n)}  & \textbf{CREDO}     & \textbf{CREDO (adap.)} & \textbf{CQR}           & \textbf{CQR-r}         & \textbf{UACQRS}        & \textbf{UACQRP}        & \textbf{EPIC}          \\ \hline
concrete (1030)          & \textbf{0.92 (0.03)} & \textbf{0.92 (0.03)} & \textbf{0.87 (0.04)} & 0.85 (0.05) & \textbf{0.88 (0.05)} & 0.86 (0.04) & \textbf{0.87 (0.05)} \\
airfoil (1503)           & \textbf{0.90 (0.04)} & \textbf{0.89 (0.03)} & \textbf{0.87 (0.04)} & \textbf{0.88 (0.05)} & \textbf{0.91 (0.04)} & \textbf{0.89 (0.04)} & \textbf{0.91 (0.03)} \\
winewhite (1599)         & \textbf{0.90 (0.01)} & \textbf{0.89 (0.02)} & 0.86 (0.02) & 0.85 (0.02) & \textbf{0.89 (0.02)} & 0.88 (0.01) & 0.88 (0.02) \\
star (2161)              & \textbf{0.88 (0.02)} & 0.87 (0.02) & 0.79 (0.04) & 0.74 (0.04) & 0.85 (0.03) & 0.85 (0.02) & 0.80 (0.03) \\
winered (4898)           & \textbf{0.91 (0.03)} & \textbf{0.90 (0.03)} & 0.82 (0.04) & 0.80 (0.04) & 0.84 (0.05) & 0.84 (0.05) & 0.83 (0.04) \\
cycle (9568)             & \textbf{0.89 (0.02)} & \textbf{0.91 (0.01)} & 0.85 (0.01) & 0.84 (0.01) & 0.86 (0.02) & 0.88 (0.01) & 0.86 (0.01) \\
electric (10000)         & 0.88 (0.01) & \textbf{0.89 (0.01)} & \textbf{0.91 (0.01)} & \textbf{0.90 (0.01)} & \textbf{0.90 (0.01)} & \textbf{0.90 (0.01)} & \textbf{0.91 (0.01)} \\
meps19 (15781)           & \textbf{0.90 (0.01)} & \textbf{0.90 (0.01)} & \textbf{0.89 (0.01)} & \textbf{0.89 (0.01)} & \textbf{0.90 (0.01)} & 0.92 (0.02) & \textbf{0.89 (0.01)} \\
superconductivity (21263) & 0.92 (0.01) & \textbf{0.91 (0.01)} & 0.87 (0.01) & 0.89 (0.01) & 0.87 (0.01) & 0.88 (0.01) & 0.88 (0.01) \\
homes (21613)            & 0.92 (0.01) & 0.92 (0.01) & \textbf{0.89 (0.01)} & \textbf{0.90 (0.01)} & 0.89 (0.01) & \textbf{0.90 (0.01)} & \textbf{0.90 (0.01)} \\
protein (45730)          & \textbf{0.90 (0.01)} & \textbf{0.90 (0.01)} & 0.89 (0.01) & 0.89 (0.01) & 0.89 (0.01) & 0.89 (0.01) & 0.89 (0.01) \\
WEC (54000)              & \textbf{0.86 (0.01)} & 0.84 (0.01) & 0.85 (0.01) & 0.85 (0.01) & 0.85 (0.01) & 0.85 (0.01) & 0.84 (0.01) \\ \hline
\end{tabular}
\end{adjustbox}
\end{table}

\subsection{Uncertainty decomposition in real data}
\label{appendix:decomp_real_data}
For the uncertainty decomposition experiments, we employ the adaptive \ourmethod{} version detailed in Section \ref{appendix:adaptive_gamma} and categorize test observations using the density-based procedure described in Section \ref{appendix:outlier_detection}. We evaluate the effectiveness of this decomposition by comparing the percentage of epistemic uncertainty—averaged by observations across 30 experimental repetitions—between the identified outliers and the top 10\% of inliers.

Figure \ref{fig:credo_boxplots_epistemic} presents the resulting boxplots for all 12 regression benchmarks. We observe that outliers consistently exhibit higher normalized epistemic uncertainty than inliers in smaller datasets, such as \textit{concrete, winewhite, winered,} and \textit{cycle}. This trend aligns with the expectation that smaller sample sizes are more susceptible to regional sparsity, leading the credal envelope to correctly assign higher model ambiguity to isolated points. Conversely, in several high-dimensional or larger datasets like \textit{homes, proteins,} and \textit{WEC}, the differences between groups remain notable, albeit sometimes more subtle. This suggests that even with larger global sample sizes, the increased dimensionality preserves local regions of scarcity where epistemic uncertainty remains a dominant factor.

\begin{figure}[h]
    \centering
    \includegraphics[width=\linewidth]{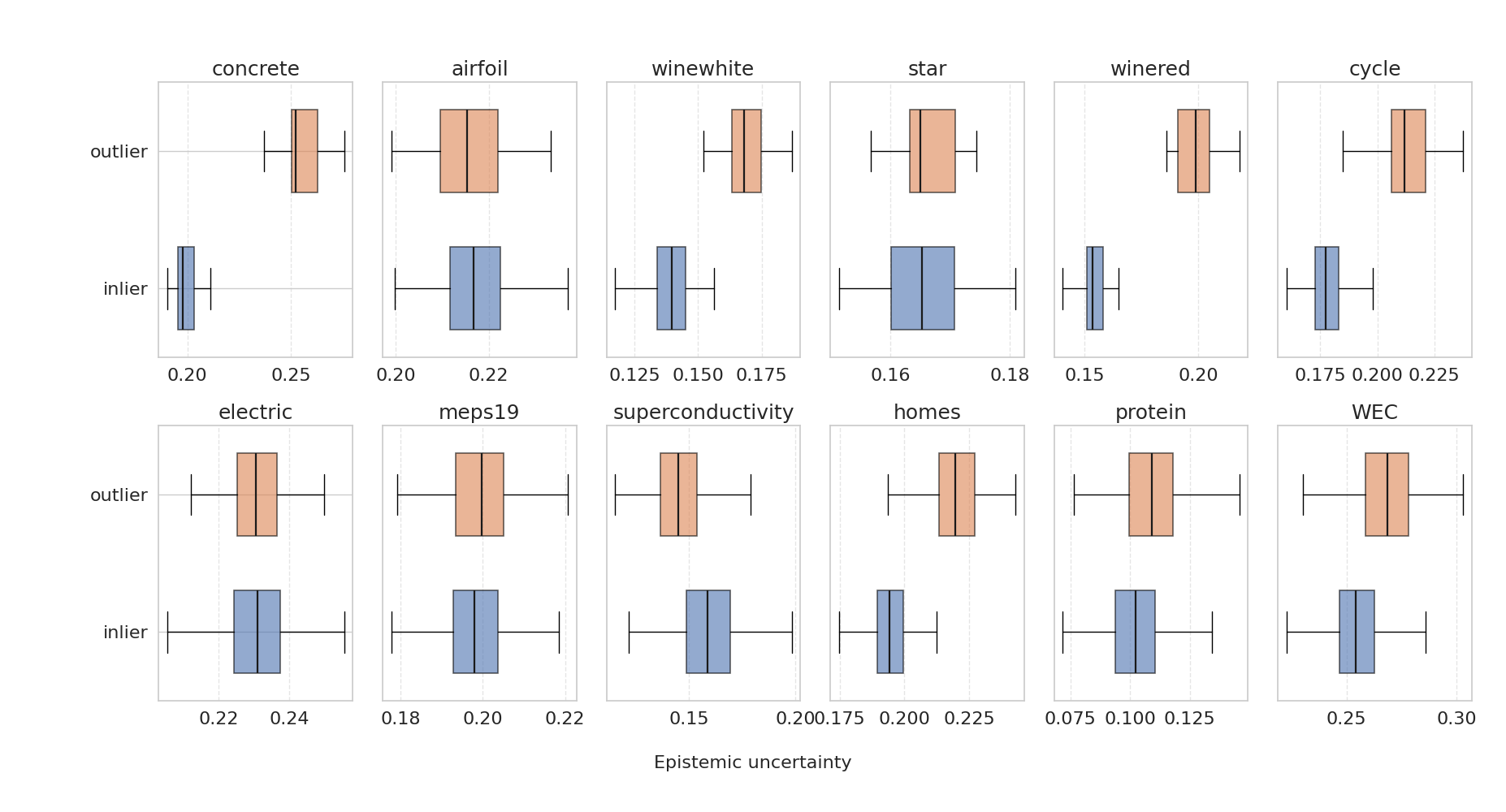}
    \caption{Percentage of Epistemic Uncertainty for outliers and inliers across real-world datasets using the QNN-based implementation of \ourmethod{}. The uncertainty decomposition successfully attributes a higher proportion of epistemic uncertainty to outliers compared to inliers, a distinction that is particularly pronounced in small-data domains (top row).}
    \label{fig:credo_boxplots_epistemic}
\end{figure}

In instances where the disparity between inliers and outliers is negligible, the result can often be attributed to model specification constraints that induce uniform epistemic behavior regardless of local data density. Such cases provide a valuable diagnostic for the base model’s representational limits: if the underlying architecture cannot distinguish between high-support and low-support regions, the credal envelope reflects this lack of discriminative power. Ultimately, these results validate the ability of \ourmethod{} to treat epistemic uncertainty as a spatially-aware signal, effectively identifying when prediction intervals are widened due to a genuine lack of information rather than inherent statistical noise.

\subsection{Simulated scenarios and Uncertainty decompositions}
\label{appendix:decomp_sim_data}
The simulated scenarios used in this work are specifically engineered to highlight the behavior of the credal envelope in regions of high data scarcity while simultaneously testing its stability in densely populated, heteroscedastic regions. We detail below the data-generating processes for the three scenarios used to evaluate the interpretability of \ourmethod{}:
\begin{itemize}
    \item \textbf{Scenario 1} (Epistemic Gap): Features a piecewise non-linear mean function where the left and right domains are well-sampled, but the middle region ($x \in [-0.2, 0.2]$) contains very few points (1\% of the samples). The noise is kept constant and small ($\sigma=0.1$) to ensure that any interval widening in the gap is strictly attributed to epistemic uncertainty.
    \item \textbf{Scenario 2} (Heteroscedastic Scarcity): Defines a scarce region in the interval $[0, 0.4]$ containing only $2\%$ of the total samples. The noise profile is defined by a complex heteroscedastic function $\sigma(x)$ that oscillates across four distinct regions: low variance ($\sigma=0.1$) for $x \le -0.3$, moderate oscillating variance ($\sigma(x) = 0.2 + 0.15 |\sin(10x)|$) for $x \in (-0.3, 0)$, high oscillating variance ($\sigma(x) = 0.7 + 0.3 |\sin(12x)|$) in the scarce region, and back to moderate levels for $x > 0.4$.
    \item \textbf{Scenario 3} (Mixture Gaps): Uses a mixture of Gaussian clusters and uniform bands to create several disjoint dense and sparse regions: a dense left cluster ($\mu=-1.8, \sigma=0.35$), a dense right cluster ($\mu=1.6, \sigma=0.45$), a sparse middle band ($x \in [-0.4, 0.4]$), and extreme tail points for extrapolation pressure. The noise is sampled from a conditional mixture where a "shock" component (shifted, high-variance) occurs with a probability $p_{shock}(x)$ that increases significantly in the far tails and right regions, governed by a logistic sigmoid $p_{shock}(x) = 0.05 + 0.35/(1+e^{-(x-1.3)}) + 0.15 \cdot \mathbb{I}(|x| > 3.2)$.
\end{itemize}
In the main text, Scenario 1 is utilized in Figure \ref{fig:credo_vs_cqr} to evaluate the capture of epistemic uncertainty and contrast the geometric behavior of \ourmethod{} against standard CQR. Scenario 2 is featured in Figure \ref{fig:gamma} to demonstrate the efficiency of the adaptive $\gamma(x)$ scheme, showcasing its ability to produce more informative intervals in well-supported regions compared to a fixed hyperparameter.

To further validate the interpretability of our framework, we provide additional displays of the uncertainty decomposition for Scenarios 2 and 3 in Figure \ref{fig:uncertainty_comparison}. These results demonstrate that \ourmethod{} successfully isolates the aleatoric core—even when that core is highly heteroscedastic or multi-modal—from the epistemic inflation triggered by the scarcity score. This confirms that the framework can distinguish between "risk" (known aleatoric noise) and "ambiguity" (lack of local information) in a consistent way.

\begin{figure}[h]
     \centering
     \begin{subfigure}[h]{0.48\linewidth}
         \centering
         \includegraphics[width=\linewidth]{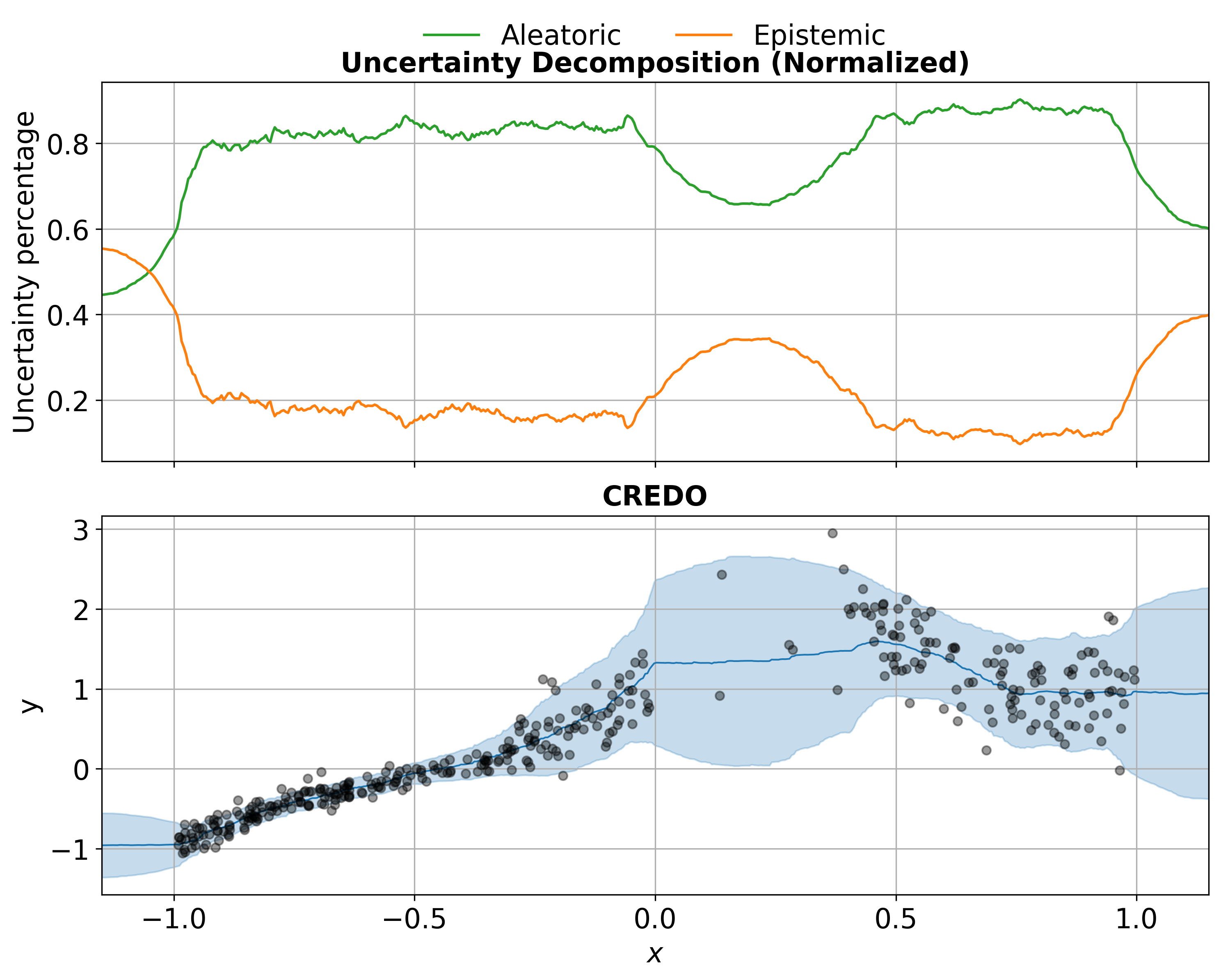}
         \caption{Decomposition in scenario 2}
         \label{fig:credo_v1_decomp}
     \end{subfigure}
     \hfill 
     \begin{subfigure}[h]{0.48\linewidth}
         \centering
         \includegraphics[width=\linewidth]{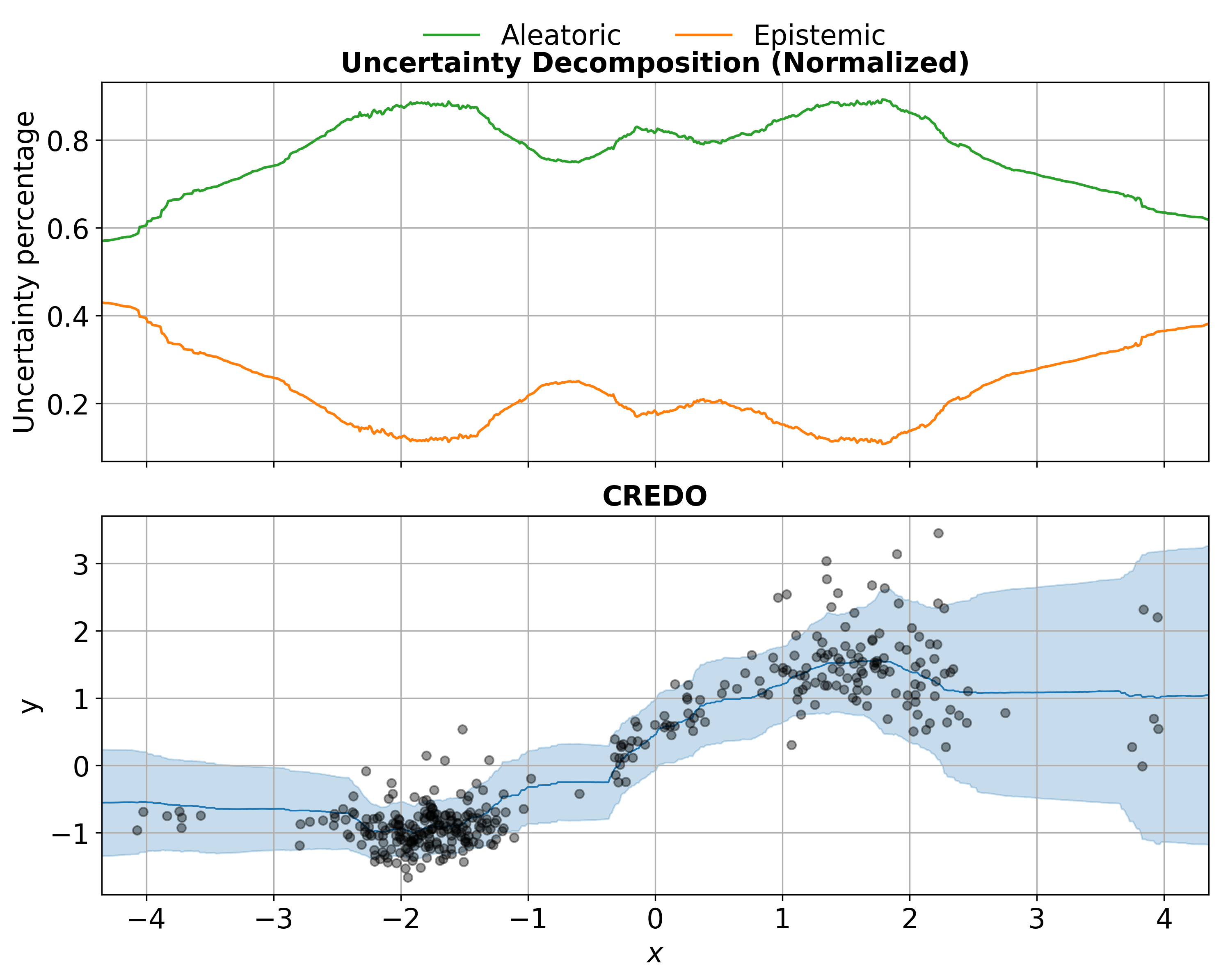}
         \caption{Decomposition in scenario 3}
         \label{fig:credo_v2_decomp}
     \end{subfigure}
     
     \caption{Normalized Uncertainty Decomposition for Scenarios 2 and 3. This figure illustrates the ability of \ourmethod{} to disentangle aleatoric and epistemic components under complex data-generating processes. In Scenario 2 (left), the decomposition successfully isolates the high epistemic inflation in the scarce region ($x \in [0, 0.4]$) from the underlying oscillating heteroscedastic noise. Similarly, in Scenario 3 (right), the epistemic percentage peaks in the gaps between clusters and in the extrapolative tails, despite the presence of multi-modal "shock" noise components.}
     \label{fig:uncertainty_comparison}
\end{figure}

\section{\ourmethod{} computational details}
\label{appendix:credo_details}
In this section, we give details on the architecture and the approach to derive the credal envelopes by using a Quantile Neural Network with dropout, as well as the standard parameter choices and heuristics for the adaptive gamma scheme. We also shed light on an alternative approach to building the credal envelopes by using BART, which was employed in some illustrations and toy examples of this work.

\subsection{QNN architecture and envelope}\label{appendix:qnn_architecture}For the deep learning-based implementation of the credal envelope, we utilize a Quantile Neural Network (QNN) designed to produce the initial quantile estimates while capturing model uncertainty via dropout.

\begin{itemize}\item \textbf{Network Architecture}: The model consists of three fully connected hidden layers with 128, 128, and 64 nodes, respectively. Each hidden layer is integrated with a batch normalization layer to stabilize training and a dropout layer (rate of 0.3) to facilitate epistemic uncertainty estimation. These dropout layers remain active during inference to allow for stochastic sampling of the network weights.

\item \textbf{Training and Optimization}: The network is trained for a maximum of 2000 epochs with an early stopping patience of 50 iterations. We apply a weight decay of $1 \times 10^{-6}$ and utilize a learning rate of $0.001$ with a batch size depending on the dataset scale, fixing 32 for small samples (less than 10000), 120 for larger datasets (above 10000), and 250 specifically for the \textit{WEC} benchmark.

\item \textbf{Envelope Construction}: To construct the credal envelope, we utilize Monte Carlo Dropout to generate $N_{MC}=1000$ samples from the posterior distribution of the quantiles. By maintaining stochastic dropout at test time, we obtain a distribution of quantile estimates that reflects the inherent model ambiguity. For a given test point $x$, the credal envelope $[\underline{C}(x), \overline{C}(x)]$ is derived by computing the $\gamma(x)/2$ and $1-\gamma(x)/2$ quantiles of these MC samples, respectively. This procedure ensures that the envelope boundaries are directly informed by the epistemic uncertainty captured within the model's posterior, allowing the prediction intervals to adapt to regions of high ambiguity.
\end{itemize}

\subsection{Adaptive Gamma standard parameter choices}
\label{appendix:adaptive_gamma}
The width of the credal envelope in \ourmethod{} is determined by the parameter $\gamma$, which dictates which quantiles are extracted from the model's posterior distribution to represent model ambiguity. While $\gamma$ can be a fixed global hyperparameter, our framework features an adaptive scheme that scales $\gamma(x)$ based on the local geometry of the feature space. This configuration utilizes the scarcity score $sc(x)$ defined in Eq. \eqref{eq:scarcity_score}. 

To maintain statistical consistency across varying geometries, we employ a heuristic for the number of neighbors $k$ that accounts for the training sample size $n$ and input dimensionality $d$ \citep{zhao2022analysis}:
$$ k = \lceil C_{base} \cdot n^{\frac{4}{4+d}} \rceil $$
where $C_{base} = 6.672$. This formulation ensures that $k$ scales at an optimal rate to balance bias and variance, preventing the scarcity metric from collapsing in high-dimensional spaces. The scarcity score is mapped to $\gamma(x)$ to dictate the "credalness" of the prediction, constrained by the defaults $\gamma_{min} = 0.1$ and $\gamma_{max} = 0.75$.

These thresholds are specifically chosen to balance the tradeoff between informativeness and the potential saturation of the envelopes in practice. Specifically, $\gamma_{max}$ establishes the baseline for epistemic uncertainty in high-density regions, ensuring the envelope maintains a minimum informative width and does not collapse into a single point. Conversely, $\gamma_{min}$ acts as a lower bound for the quantile levels in sparse regions; by reaching this limit, the envelope expands to capture the extreme tails of the posterior without becoming non-informative or excessively wide. 

Additionally, by default, we set the scarcity threshold $m_{\gamma} = 0$ and the scaling parameter $\tau_{\gamma} = 1$. To ensure the robustness of the mapping, we incorporate a small numerical constant $\epsilon = 1 \times 10^{-6}$ within the scarcity calculations, providing stability in highly clustered regions of the feature space. Furthermore, the representation $\phi(x)$ is defined as the standardized features to maintain a consistent scale across different datasets.

\subsection{BART-Based Credal Envelope}
\label{appendix:bart_envelope}

As a non-parametric alternative to neural architectures, we leverage Bayesian Additive Regression Trees (BART) to derive the credal envelope. While BART provides a more formal Bayesian posterior compared to dropout-based approximations, it was primarily utilized in our toy examples and illustrations. Due to the high computational cost of MCMC sampling, the Quantile Neural Network (QNN) architecture was preferred for the real-world benchmarks, as it offered better scalability without sacrificing predictive performance.

\begin{itemize}

\item \textbf{Model Specification}: We utilize a heteroscedastic formulation of BART, where both the conditional mean and the residual variance are modeled as functions of the covariates. Specifically, the model assumes $Y = f(x) + \epsilon(x)$, where $f(x) = \sum_{j=1}^m T_j(x; M_j)$ is a sum of regression trees capturing the mean, and the noise term $\epsilon(x) \sim \mathcal{N}(0, \sigma^2(x))$ follows a variance process also governed by a secondary tree ensemble. This allows the model to capture local variations in aleatoric noise while the Bayesian prior over the tree structures and leaf parameters accounts for epistemic uncertainty.

\item \textbf{MCMC Sampling}: We employ Markov Chain Monte Carlo (MCMC) to generate samples from the posterior distribution of the regression function. These chains provide a rich, non-parametric representation of model ambiguity, as each iteration in the chain reflects a valid hypothesis for the underlying data-generating process given the evidence.

\item \textbf{Envelope Construction}: For a test point $x$, we construct the credal envelope by leveraging the full posterior distribution captured through MCMC sampling. Specifically, we utilize the tree-based model structure to generate a collection of quantile estimates from each iteration of the posterior chain. The final boundaries $[\underline{C}(x), \overline{C}(x)]$ are then defined by extracting the empirical $\gamma(x)/2$ and $1-\gamma(x)/2$ quantiles across this distribution of quantile estimates. This hierarchical approach ensures the envelope reflects the model's epistemic uncertainty by directly quantifying the variance among the different candidate models stored in the MCMC chains.
\end{itemize}

\section{Proofs}

\begin{lemma}[Quantile sandwich]
\label{lem:quantile-sandwich}
Let $Y\sim F$ where $F$ is a CDF, and let $q=F^{-1}(p)$ for $p\in(0,1)$.
Then
\[
\mathbb{P}(Y<q)\le p
\qquad\text{and}\qquad
\mathbb{P}(Y\le q)\ge p.
\]
Consequently, for any $0<p<1/2$,
\[
\mathbb{P}\bigl(F^{-1}(p)\le Y \le F^{-1}(1-p)\bigr)\ge 1-2p.
\]
\end{lemma}

\begin{proof}
By definition of $q=F^{-1}(p)$, we have $F(t)<p$ for all $t<q$ and $F(q)\ge p$.
Thus $\mathbb{P}(Y<q)=\sup_{t<q}\mathbb{P}(Y\le t)=\sup_{t<q}F(t)\le p$.
Also $\mathbb{P}(Y\le q)=F(q)\ge p$. The last display follows by a union bound:
\[
\mathbb{P}\bigl(Y< F^{-1}(p)\bigr)\le p,
\qquad
\mathbb{P}\bigl(Y>F^{-1}(1-p)\bigr)=1-\mathbb{P}\bigl(Y\le F^{-1}(1-p)\bigr)\le p.
\]
\end{proof}

\begin{proof}[Proof of Theorem \ref{thm:credal-coverage-envelope}]
Fix $x$ and $F\in\mathcal F_0(x)$, and define the $F$-central endpoints
\[
q_L(F,x):=F^{-1}(\alpha_0/2\mid x),
\qquad
q_U(F,x):=F^{-1}(1-\alpha_0/2\mid x).
\]
By definition of $\ell(x)$ as an infimum over $F$-quantiles,
$\ell(x)\le q_L(F,x)$. By definition of $u(x)$ as a supremum over $F$-quantiles,
$u(x)\ge q_U(F,x)$. Hence
\[
[\ell(x),u(x)] \supseteq [q_L(F,x),q_U(F,x)].
\]
Applying Lemma~\ref{lem:quantile-sandwich} conditionally on $X=x$ gives
\[
\mathbb{P}_F\!\left(q_L(F,x)\le Y\le q_U(F,x)\mid X=x\right)\ge 1-\alpha_0.
\]
Because $[\ell(x),u(x)]$ contains $[q_L(F,x),q_U(F,x)]$, the same lower bound
holds for $[\ell(x),u(x)]$.
\end{proof}

\begin{proof}[Proof of Theorem \ref{theorem:posterior-predictive-trimmed}]
By construction of $C_L(x)$ and $C_U(x)$, we have
\[
\pi\!\left(q_L(\theta,x)\le C_L(x)\mid D_{\mathrm{tr}}\right)=\gamma/2,
\qquad
\pi\!\left(q_U(\theta,x)\ge C_U(x)\mid D_{\mathrm{tr}}\right)=\gamma/2.
\]
Hence
\[
\pi\!\left(q_L(\theta,x)\ge C_L(x)\mid D_{\mathrm{tr}}\right)\ge 1-\gamma/2,
\qquad
\pi\!\left(q_U(\theta,x)\le C_U(x)\mid D_{\mathrm{tr}}\right)\ge 1-\gamma/2.
\]
Let
$R_{qq}(x)=\{\theta:\ q_L(\theta,x)\ge C_L(x),\ q_U(\theta,x)\le C_U(x)\}$.
By the union bound,
\[
\pi\!\left(\theta\in R_{qq}(x)\mid D_{\mathrm{tr}}\right)
\ge 1-\gamma.
\]

Now fix any $\theta\in R_{qq}(x)$. Then
$[q_L(\theta,x),q_U(\theta,x)]\subseteq[C_L(x),C_U(x)]=[\ell(x),u(x)]$, and therefore
\[
\P_\theta\!(Y\in[\ell(x),u(x)]\mid X=x)
\ \ge\
\P_\theta\!(Y\in[q_L(\theta,x),q_U(\theta,x)]\mid X=x).
\]
The interval $[q_L(\theta,x),q_U(\theta,x)]$ is the central $(1-\alpha_0)$
quantile interval under $F_\theta(\cdot\mid x)$, so it satisfies
\[
\P_\theta\!(Y\in[q_L(\theta,x),q_U(\theta,x)]\mid X=x )\ge 1-\alpha_0,
\]
and thus, for all $\theta\in R_{qq}(x)$,
\[
\P_\theta (Y\in[\ell(x),u(x)]\mid X=x) \ge 1-\alpha_0.
\]

Finally, by the definition of the posterior predictive probability,
\begin{align*}
\P\!(Y\in[\ell(x),u(x)] \mid X=x, D_{\mathrm{tr}}) 
&=
\int \P_\theta\! ( Y\in[\ell(x),u(x)]\mid X=x ) \, d\pi(\theta\mid D_{\mathrm{tr}})\\
&\ge
\int_{R_{qq}(x)} \P_\theta\!( Y\in[\ell(x),u(x)]\mid X=x ) \, d\pi(\theta\mid D_{\mathrm{tr}})\\
&\ge
(1-\alpha_0)\,\pi\!\left(\theta\in R_{qq}(x)\mid D_{\mathrm{tr}}\right)\\
&\ge (1-\alpha_0)(1-\gamma),
\end{align*}
which proves the claim.
\end{proof}

\begin{proof}[Proof of Theorem \ref{thm:split-conformal}]
Follows from the standard split-conformal arguments \cite{Vovk2005,Lei2018,vovk2022algorithmic,Izbicki2025}.
    
\end{proof}

\begin{assumption}[Correct specification and posterior consistency]
\label{assump:oracle}
Let $\alpha_0=\alpha\in(0,1)$. Assume there exists a unique true parameter
$\theta^\star\in\Theta$ such that the data are generated i.i.d.\ from the joint
distribution induced by $\theta^\star$, i.e.\ $(X_i,Y_i)_{i\ge 1}$ are i.i.d.\ and
for $P_X$-almost every $x$,
\[
Y\mid X=x \sim F_{\theta^\star}(\cdot\mid x).
\]
Let $D_{\mathrm{tr}}$ and $D_{\mathrm{cal}}$ be obtained by sample splitting with
$|D_{\mathrm{tr}}|=n_{\mathrm{tr}}\to\infty$ and $|D_{\mathrm{cal}}|=m\to\infty$.
All limits below are with respect to the joint draw of the full dataset (and,
if applicable, the random split).

\smallskip
We assume:

\begin{enumerate}
\item[(A1)] \textbf{Posterior consistency for $\theta$ (in probability).}
There exists a metric $d$ on $\Theta$ such that for every $\varepsilon>0$,
\[
\pi\!\left(d(\theta,\theta^\star)>\varepsilon\mid D_{\mathrm{tr}}\right)
\ \xrightarrow{\ \mathbb P_{\theta^\star}\ }\ 0.
\]

\item[(A2)] \textbf{Continuity of the endpoint functionals at $\theta^\star$.}
For $P_X$-almost every $x$, define the endpoint maps (at nominal level $\alpha$)
\[
q_L(\theta,x)=F_\theta^{-1}(\alpha/2\mid x),
\qquad
q_U(\theta,x)=F_\theta^{-1}(1-\alpha/2\mid x).
\]
Assume $q_L(\theta,x)$ and $q_U(\theta,x)$ are continuous in $\theta$ at
$\theta^\star$ with respect to $d$.

\item[(A3)] \textbf{Regularity of the true conditional distribution at the oracle endpoints.}
For $P_X$-almost every $x$, let
\[
q_L^\star(x)=F_{\theta^\star}^{-1}(\alpha/2\mid x),
\qquad
q_U^\star(x)=F_{\theta^\star}^{-1}(1-\alpha/2\mid x).
\]
Assume that $F_{\theta^\star}(\cdot\mid x)$ is continuous and strictly
increasing in a neighborhood of $q_L^\star(x)$ and $q_U^\star(x)$ (equivalently,
no atoms at the endpoints and no flat spots locally).

\end{enumerate}
\end{assumption}

\begin{proof}[Proof of Theorem \ref{thm:oracle}]

    \textbf{Step 1: posterior concentration transfers to the endpoints.}
Fix $\delta>0$. By continuity of $q_L(\theta,x)$ at $\theta^\star$ (Assumption
\ref{assump:oracle}(A2)), there exists $\varepsilon(\delta)>0$ such that
$d(\theta,\theta^\star)\le \varepsilon(\delta)$ implies
$|q_L(\theta,x)-q_L^\star(x)|\le \delta$. Therefore,
\[
\Bigl\{\theta:|q_L(\theta,x)-q_L^\star(x)|>\delta\Bigr\}
\subseteq
\Bigl\{\theta:d(\theta,\theta^\star)>\varepsilon(\delta)\Bigr\}.
\]
Taking posterior probabilities conditional on $D_{\mathrm{tr}}$ gives
\[
\pi\!\left(|q_L(\theta,x)-q_L^\star(x)|>\delta\mid D_{\mathrm{tr}}\right)
\le
\pi\!\left(d(\theta,\theta^\star)>\varepsilon(\delta)\mid D_{\mathrm{tr}}\right).
\]
By posterior consistency (Assumption \ref{assump:oracle}(A1)), the right-hand
side converges to $0$ in $\mathbb P_{\theta^\star}$-probability, hence
\begin{equation}
\label{eq:pushforward-ql}
\pi\!\left(|q_L(\theta,x)-q_L^\star(x)|>\delta\mid D_{\mathrm{tr}}\right)
\xrightarrow{\ \mathbb P_{\theta^\star}\ }0.
\end{equation}
The same argument yields
\begin{equation}
\label{eq:pushforward-qu}
\pi\!\left(|q_U(\theta,x)-q_U^\star(x)|>\delta\mid D_{\mathrm{tr}}\right)
\xrightarrow{\ \mathbb P_{\theta^\star}\ }0.
\end{equation}
\\

\textbf{Step 2: posterior quantiles of concentrated random variables converge.}
We show $\ell(x)=C_L(x)\to q_L^\star(x)$ in probability; the proof for $u(x)$
is identical. 
Let $Z=q_L(\theta,x)$ under $\theta\sim\pi(\cdot\mid D_{\mathrm{tr}})$, and let
$C_L(x)$ be its $(\gamma(x)/2)$-quantile under that posterior.
Fix $\delta>0$. If the posterior puts at least $1-\eta$ mass on the interval
$[q_L^\star(x)-\delta,\ q_L^\star(x)+\delta]$ for some $\eta<\gamma(x)/2$, then
necessarily the $(\gamma(x)/2)$-quantile lies inside that interval, i.e.
$|C_L(x)-q_L^\star(x)|\le \delta$. Formally, on the event
\[
E_{\delta,\eta}
:=
\left\{
\pi\!\left(|Z-q_L^\star(x)|\le \delta\mid D_{\mathrm{tr}}\right)\ge 1-\eta
\right\},
\quad \text{ with }   \eta<\gamma(x)/2,
\]
we have $\pi(Z\le q_L^\star(x)-\delta\mid D_{\mathrm{tr}})\le \eta<\gamma(x)/2$,
so the $(\gamma(x)/2)$-quantile cannot be smaller than $q_L^\star(x)-\delta$.
Similarly, $\pi(Z\le q_L^\star(x)+\delta\mid D_{\mathrm{tr}})\ge 1-\eta >
\gamma(x)/2$, so the quantile cannot exceed $q_L^\star(x)+\delta$.
Hence, on $E_{\delta,\eta}$ we have $|C_L(x)-q_L^\star(x)|\le \delta$.

By \eqref{eq:pushforward-ql}, for any fixed $\eta>0$,
\[
\pi\!\left(|Z-q_L^\star(x)|\le \delta\mid D_{\mathrm{tr}}\right)
=1-\pi\!\left(|Z-q_L^\star(x)|> \delta\mid D_{\mathrm{tr}}\right)
\xrightarrow{\ \mathbb P_{\theta^\star}\ }1,
\]
so $\mathbb P_{\theta^\star}(E_{\delta,\eta})\to 1$. Therefore,
$\mathbb P_{\theta^\star}(|\ell(x)-q_L^\star(x)|>\delta)\to 0$, i.e.
$\ell(x)\to q_L^\star(x)$ in $\mathbb P_{\theta^\star}$-probability.
Repeating the same argument with \eqref{eq:pushforward-qu} yields
$u(x)\to q_U^\star(x)$, proving (i).
\\

\textbf{Step 3: $\hat\tau\to 0$.}
Let $(X,Y)$ denote an independent draw from the true distribution
$\mathbb P_{\theta^\star}$, independent of $D_{\mathrm{tr}}$.
Define the \emph{oracle} interval and score as
\[
I^\star(X):=[q_L^\star(X),q_U^\star(X)] \text{ and }
s^\star(X,Y):=d\!\left(Y,I^\star(X)\right),
\]
and the learned score (based on $D_{\mathrm{tr}}$ only) as
\[
s(X,Y):=d\!\left(Y,[\ell(X),u(X)]\right).
\]

\smallskip
\noindent\textbf{(a) The oracle score has $(1-\alpha)$-quantile equal to $0$.}
Fix $\varepsilon>0$. By Assumption~\ref{assump:oracle}(A3), for every $x$,
the map $t\mapsto F_{\theta^\star}(t\mid x)$ is continuous and strictly
increasing in a neighborhood of $q_L^\star(x)$ and $q_U^\star(x)$.
Therefore expanding the oracle interval by $\varepsilon$ increases its
conditional probability strictly:
\[
\mathbb P_{\theta^\star}\!\bigl(s^\star(X,Y)\le \varepsilon \mid X=x\bigr)
=
\mathbb P_{\theta^\star}\!\bigl(
Y\in[q_L^\star(x)-\varepsilon,\ q_U^\star(x)+\varepsilon]\mid X=x
\bigr)
> 1-\alpha.
\]
Taking expectation over $X$ yields
\begin{equation}
\label{eq:oracle-score-eps}
p^\star(\varepsilon)
:=
\mathbb P_{\theta^\star}\!\bigl(s^\star(X,Y)\le \varepsilon\bigr)
=
\mathbb E\!\left[
\mathbb P_{\theta^\star}\!\bigl(s^\star(X,Y)\le \varepsilon\mid X\bigr)
\right]
>1-\alpha.
\end{equation}

\smallskip
\noindent\textbf{(b) The learned score converges to the oracle score.}
Let $(X,Y)\sim \mathbb P_{\theta^\star}$ be independent of $D_{\mathrm{tr}}$.
By part~(i), for every $\delta>0$ we have, for $P_X$-a.e.\ $x$,
$\mathbb P_{\theta^\star}(|\ell(x)-q_L^\star(x)|>\delta)\to 0$ and
$\mathbb P_{\theta^\star}(|u(x)-q_U^\star(x)|>\delta)\to 0$.
Since these probabilities are bounded by $1$, dominated convergence implies
\[
\mathbb P_{\theta^\star}\big(|\ell(X)-q_L^\star(X)|>\delta\big)
=
\int \mathbb P_{\theta^\star}\big(|\ell(x)-q_L^\star(x)|>\delta\big)\,dP_X(x)
\ \longrightarrow\ 0,
\]
and similarly $\mathbb P_{\theta^\star}(|u(X)-q_U^\star(X)|>\delta)\to 0$.
Thus $\ell(X)\to q_L^\star(X)$ and $u(X)\to q_U^\star(X)$ in
$\mathbb P_{\theta^\star}$-probability.
Moreover, for any $y$ and any intervals $[a,b]$ and $[a',b']$,
\[
\bigl|d(y,[a,b]) - d(y,[a',b'])\bigr|
\le |a-a'|+|b-b'|.
\]
Applying this with $a=\ell(X)$, $b=u(X)$, $a'=q_L^\star(X)$, and
$b'=q_U^\star(X)$ yields
\[
|s(X,Y)-s^\star(X,Y)|
\le |\ell(X)-q_L^\star(X)| + |u(X)-q_U^\star(X)|
\xrightarrow{\ \mathbb P_{\theta^\star}\ } 0.
\]
Therefore 
\begin{align}
    \label{eq:score-conv}  s(X,Y)\xrightarrow{\mathbb P_{\theta^\star}} s^\star(X,Y).
\end{align}

\smallskip
\noindent\textbf{(c) The conditional score CDF at $\varepsilon$ exceeds $1-\alpha$ with high probability.}
Fix $\varepsilon>0$ and set $\varepsilon'=\varepsilon/2$.
On the event $\{|s(X,Y)-s^\star(X,Y)|\le \varepsilon'\}$,
the inclusion $\{s^\star(X,Y)\le \varepsilon'\}\subseteq\{s(X,Y)\le \varepsilon\}$
holds, hence
\[
\mathbb P_{\theta^\star}\!\bigl(s(X,Y)\le \varepsilon \mid D_{\mathrm{tr}}\bigr)
\ \ge\
\mathbb P_{\theta^\star}\!\bigl(s^\star(X,Y)\le \varepsilon'\bigr)
\ -\
\mathbb P_{\theta^\star}\!\bigl(|s(X,Y)-s^\star(X,Y)|>\varepsilon'\mid D_{\mathrm{tr}}\bigr).
\]
By \eqref{eq:oracle-score-eps}, the first term equals $p^\star(\varepsilon')>1-\alpha$.
By \eqref{eq:score-conv} and  Markov's inequality,
\[
\mathbb P_{\theta^\star}\!\bigl(|s(X,Y)-s^\star(X,Y)|>\varepsilon'\mid D_{\mathrm{tr}}\bigr)
\ \xrightarrow{\ \mathbb P_{\theta^\star}\ }\ 0.
\]
Therefore,
\begin{equation}
\label{eq:cond-cdf-gt}
\mathbb P_{\theta^\star}\!\bigl(s(X,Y)\le \varepsilon \mid D_{\mathrm{tr}}\bigr)
\ >\ 1-\alpha
\qquad\text{with probability tending to $1$.}
\end{equation}

\smallskip
\noindent\textbf{(d) Conclude $\hat\tau\to 0$.}
Let $s_1,\dots,s_m$ be the calibration scores and define
$\hat F_m(t)=\frac1m\sum_{i=1}^m\mathbb I\{s_i\le t\}$.
Conditional on $D_{\mathrm{tr}}$, the $s_i$'s are i.i.d., so by the weak law,
for any $\eta>0$,
\[
\hat F_m(\varepsilon)
\ -\
\mathbb P_{\theta^\star}\!\bigl(s(X,Y)\le \varepsilon \mid D_{\mathrm{tr}}\bigr)
\ \xrightarrow{\ \mathbb P_{\theta^\star}\ }\ 0
\quad\text{as }m\to\infty.
\]
Combine this with \eqref{eq:cond-cdf-gt}. Since $k=\lceil(m+1)(1-\alpha)\rceil$
satisfies $(k-1)/m \to 1-\alpha$, the event
$\{\hat F_m(\varepsilon)>(k-1)/m\}$ holds with probability tending to $1$, and
on this event at least $k$ scores are $\le\varepsilon$, implying the $k$-th
order statistic satisfies $\hat\tau\le\varepsilon$.
Thus for every $\varepsilon>0$,
\[
\mathbb P_{\theta^\star}(\hat\tau>\varepsilon)\to 0,
\]
i.e.\ $\hat\tau\xrightarrow{\mathbb P_{\theta^\star}}0$.

\textbf{Step 4: conditional coverage converges to $1-\alpha$.}
Fix $x$. The conditional coverage of the final interval is
\[
\mathbb P_{\theta^\star}(Y\in C(x)\mid X=x,\ D_{\mathrm{tr}},D_{\mathrm{cal}})
=
F_{\theta^\star}(u(x)+\hat\tau\mid x)\ -\ F_{\theta^\star}(\ell(x)-\hat\tau\mid x),
\]
using continuity of $F_{\theta^\star}(\cdot\mid x)$ (Assumption \ref{assump:oracle}(A3)).
By (i) and (ii) and Slutsky's theorem,
\[
\bigl(\ell(x)-\hat\tau,\ u(x)+\hat\tau\bigr)
\ \xrightarrow{\ \mathbb P_{\theta^\star}\ }\ 
\bigl(q_L^\star(x),\ q_U^\star(x)\bigr).
\]
Because $F_{\theta^\star}(\cdot\mid x)$ is continuous, the map
$(a,b)\mapsto F_{\theta^\star}(b\mid x)-F_{\theta^\star}(a\mid x)$ is continuous,
so by the continuous mapping theorem,
\[
F_{\theta^\star}(u(x)+\hat\tau\mid x)\ -\ F_{\theta^\star}(\ell(x)-\hat\tau\mid x)
\ \xrightarrow{\ \mathbb P_{\theta^\star}\ }\ 
F_{\theta^\star}(q_U^\star(x)\mid x)\ -\ F_{\theta^\star}(q_L^\star(x)\mid x).
\]
Finally, since $q_L^\star(x)$ and $q_U^\star(x)$ are the $\alpha/2$ and
$1-\alpha/2$ quantiles and the CDF is continuous at these points,
$F_{\theta^\star}(q_L^\star(x)\mid x)=\alpha/2$ and
$F_{\theta^\star}(q_U^\star(x)\mid x)=1-\alpha/2$, hence the limit is $1-\alpha$.
This proves (iii).
\end{proof}

\end{document}